\documentclass[letterpaper, 10 pt, conference]{ieeeconf}  

\IEEEoverridecommandlockouts                              
\usepackage{enumerate}
\usepackage[cmex10]{amsmath} 
\usepackage{amssymb} 
\usepackage{amsfonts}

\usepackage{amsthm}
  
\usepackage[mathcal]{euscript}

\usepackage{graphicx}
\usepackage{graphics} 
\usepackage{epsfig} 
\usepackage{accents}
\usepackage{algorithm}
\usepackage{algpseudocode}
\usepackage{subcaption}

\usepackage{color}
\makeatletter
\let\NAT@parse\undefined
\makeatother

\usepackage{cite}

\usepackage{nohyperref}
\usepackage{url}
\usepackage{comment}

	\newtheoremstyle{myplain}
	  {}
	  {}
	  {\itshape}
	  {}
	  {\bfseries}
	  {}
	  {5pt plus 1pt minus 1pt}
	  {}

	\newtheoremstyle{mydefinition}
	  {}
	  {}
	  {\normalfont}
	  {}
	  {\bfseries}
	  {}
	  {5pt plus 1pt minus 1pt}
	  {}
  
	\theoremstyle{myplain}
    \newtheorem{assumption}{Assumption}

	\newtheorem{theorem}{Theorem}
	\newtheorem{lemma}{Lemma}
	\newtheorem{proposition}{Proposition}
	\newtheorem{corollary}{Corollary}

	\theoremstyle{mydefinition}
	\newtheorem{definition}{Definition}
	\newcommand{\argmin}{\operatornamewithlimits{arg\ min}} 

\algdef{SE}[DOWHILE]{Do}{doWhile}{\algorithmicdo}[1]{\algorithmicwhile\ #1}%

\include{notation}

\title{\LARGE \bf
Technical Report: Sensor-Based Reactive Symbolic Planning in Partially Known Environments
}

\author{Vasileios Vasilopoulos, William Vega-Brown, Omur Arslan, Nicholas Roy, Daniel E. Koditschek
\thanks{Vasileios Vasilopoulos is with the Department of Mechanical Engineering and Applied Mechanics, University of Pennsylvania, Philadelphia, PA 19104,
        {\tt\small vvasilo@seas.upenn.edu}}%
\thanks{William Vega-Brown and Nicholas Roy are with the Computer Science and Artificial Intelligence Laboratory, Massachusetts Institute of Technology, Cambridge, Massachusetts 02139,
        {\tt\small \{wrvb,nickroy\}@csail.mit.edu}}%
\thanks{Omur Arslan and Daniel E. Koditschek are with the Department of Electrical and Systems Engineering, University of Pennsylvania, Philadelphia, PA 19104,
        {\tt\small \{omur,kod\}@seas.upenn.edu}}
}

\begin{document}

\maketitle
\thispagestyle{empty}
\pagestyle{empty}

\begin{abstract}
This paper considers the problem of completing assemblies of passive objects in nonconvex environments, cluttered with convex obstacles of unknown position, shape and size that satisfy a specific separation assumption. A differential drive robot equipped with a gripper and a LIDAR sensor, capable of perceiving its environment only locally, is used to position the passive objects in a desired configuration. The method combines the virtues of a deliberative planner generating high-level, symbolic commands, with the formal guarantees of convergence and obstacle avoidance of a reactive planner that requires little onboard computation and is used online. The validity of the proposed method is verified both with formal proofs and numerical simulations.
\end{abstract}


\section{INTRODUCTION}
\label{sec:introduction}

\begin{figure}[t]
\centering
\includegraphics[width=0.9\columnwidth]{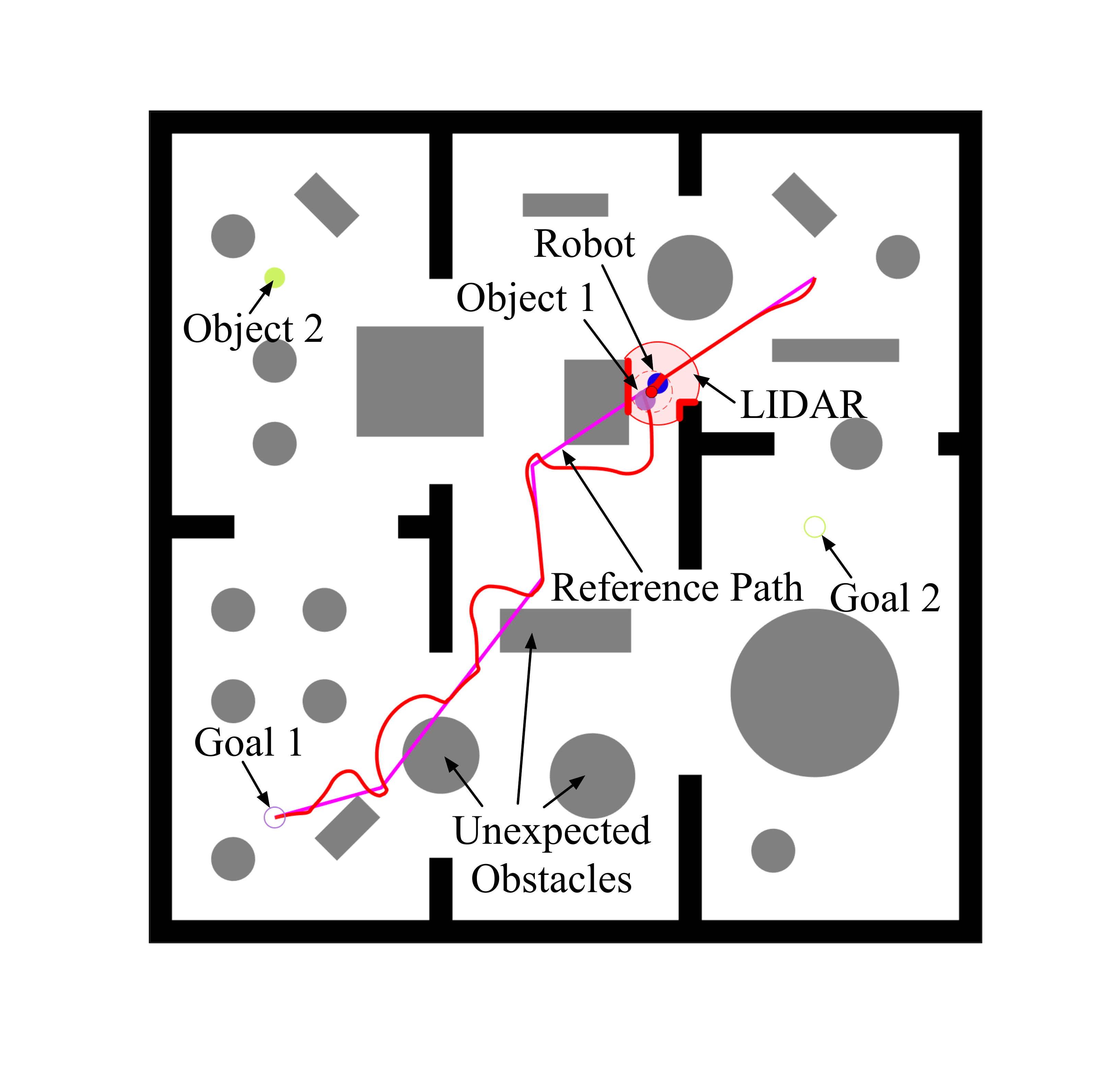}
\caption{A depiction of an intermediate stage of an assembly process. The robot is tasked to move two objects from their start to their final configuration using a gripper and a LIDAR. The deliberative planner outputs a reference path (purple) which the reactive planner has to follow, while avoiding the unexpected obstacles (grey) in the (potentially) nonconvex workspace. The resulting piecewise differentiable object trajectory for one object is shown in red.}
\label{fig:environment}
\end{figure}

In this paper, we address a very specific instance of the Warehouseman's Problem \cite{Hopcroft_Schwartz_Sharir_1984} as a challenging setting in which to advance the formal integration of deliberative and reactive modes of robot assembly planning and control. We posit a planar disk-shaped robot with velocity controlled unicycle kinematics placed in an indoor environment with known floor-plan, cluttered with convex obstacles of unknown number and placement. The robot's task is to bring a collection of  known disk-shaped objects from their initial placement to their prescribed destination by approaching, attaching and then pushing it into place, making sure to avoid any collisions with the known walls, other objects and unanticipated obstacles along the way. 

\subsection{Motivation and Related Work}
The problem of using a higher-level planner to inform subgoals of a lower-level planner has been examined previously, and we build on prior work in hybrid systems and task and motion planning. However, most work has focused on ad hoc abstractions that perform well empirically. For example, Wolfe et al. \cite{Wolfe2010} use a task hierarchy to guide the search for a low-level plan by expanding high-level plans in a best-first way. This approach guarantees hierarchical optimality: it will generate the best plan which can be represented in a given task hierarchy. Ensuring optimality has always been difficult to achieve due to computational complexity. Berenson et al \cite{Berenson2009} and  Konidaris et al \cite{Konidaris2014} use specific formulations of hierarchy without guaranteeing optimality. Kaelbling and Lozano-Perez \cite{Kaelbling2011} avoid the computational cost by committing to decisions at a high level of abstraction, before a full low-level plan is available. Vega-Brown and Roy \cite{WAFR16WRVB} provided a further step towards tractable planning that incorporated complex kinematic constraints, but no general approach exists for incorporating complex dynamics.

On the other hand, planning the rearrangement of movable objects has long been known to be algorithmically hard  (e.g., PSPACE hardness was established in \cite{Hopcroft_Schwartz_Sharir_1984 }) and a lively contemporary literature \cite{Vendittelli_Laumond_Mishra_2015,Deshpande_Kaelbling_Lozano-Perez_2016} continues to explore conditions under which the additional complexity of planning the grasps results in a deterministically undecidable problem. While that interface has been understood to be crucial for decades \cite{Chatila_1995}, the literature on reactive approaches to this problem has been far more sparse.  Purely reactive formulations and provably correct solutions to  partial \cite{Koditschek_1994} and then complete \cite{Bozma_Koditschek_2001} scalar versions of the problem motivated the empirical study \cite{Karagoz_Bozma_Koditschek_2004} of a simplified proxy of the planar problem we consider here.



\subsection{Contributions}

We present a provably correct architecture (Fig. \ref{fig:control_idea}) for planning and executing a successful solution to this  Warehouseman's problem by decomposition into an offline ``deliberative'' planning module and an online ``reactive'' execution module.  The deliberative planner, adapted from the probabilistically complete (and optimal) algorithm of \cite{marthi2008angelic}, is assigned the job of finding an assembly plan, while the reactive planner accepts each next step of that planned sequence, and uses online (LIDAR-style) sensory measurements to avoid the unanticipated obstacles (as well as the known walls and objects) by switching between following the deliberative planner's specified path or instead following a sensed  wall. The wall following algorithm is guaranteed to maintain the robot distance from the wall within some specified bounds, while making progress along the wall boundary.


After imposing specific constraints on how tightly packed  the unknown obstacles and the known objects' initial and final configurations can be, we prove that the hybrid control scheme generated by this reactive planner must succeed in achieving any specified step of the deliberative sequence with no collisions along the way.

\subsection{Organization of the paper}

The paper is organized as follows. Section \ref{sec:problemformulation} describes the problem and summarizes our approach. Section \ref{sec:angelic} gives a brief outline of the high-level deliberative planner that generates the sequence of appropriate symbolic commands to accomplish the task at hand, without any information about the internal obstacles. Section \ref{sec:reactive} describes the fundamental idea of reactively switching between a path following and a wall following mode, for both a holonomic and a nonholonomic robot, while Section \ref{sec:models} extends our reactive ideas to the navigation problem of a nonholonomic robot grasping a passive object and using its sensor to position it at a desired location. Section \ref{sec:symboliclanguage} combines the ideas from the previous two sections and describes the low-level, online implementation of the symbolic action command set. Section \ref{sec:simulations} presents illustrative numerical examples for the ideas presented, while Section \ref{sec:conclusion} summarizes our observations and ideas for future work.

\section{PROBLEM FORMULATION}
\label{sec:problemformulation}

In this work, we consider a first-order, nonholonomically-constrained, disk-shaped robot, centered at $\mathbf{x} \in \mathbb{R}^2$ with radius $r \in \mathbb{R}_{> 0}$ and orientation $\psi \in S^1$, using a gripper to move circular objects in a closed, compact, not necessarily convex workspace $\mathcal{W} \subset \mathbb{R}^2$ as shown in Fig.~\ref{fig:environment}, whose boundary $\partial \mathcal{W}$ is assumed to be known. The robot dynamics are described by 
\begin{equation}
(\dot{\mathbf{x}}, \dot{\psi}) = \mathbf{B}(\psi) \mathbf{u}_{ku} \label{eq:robotEOM}
\end{equation}
with $\mathbf{B}(\psi) = \begin{bmatrix}
\cos\psi & \sin\psi & 0 \\ 0 & 0 & 1
\end{bmatrix}^T$ the differential constraint matrix and $\mathbf{u}_{ku} = (v,\omega)$ the input vector\footnote{Throughout this paper, we will use the ordered set notation $(*,*,\ldots)$ and the matrix notation $\begin{bmatrix}
* & * & \ldots
\end{bmatrix}^T$ for vectors interchangeably.} consisting of a linear and an angular command. The robot is assumed to possess a LIDAR, positioned at $\mathbf{x}$, with a $360^\circ$ angular scanning range and a fixed sensing range $R \in \mathbb{R}_{>0}$ and is tasked with moving each of the $n \in \mathbb{N}$ movable disk-shaped objects, centered at $\mathbf{p}:=(\mathbf{p}_1, \mathbf{p}_2, \ldots, \mathbf{p}_n) \in \mathcal{W}^n$ with a vector of radii $(\mathbf{\rho}_1, \mathbf{\rho}_2, \ldots, \mathbf{\rho}_n) \in (\mathbb{R}_{> 0})^n$, from its initial configuration to a user-specified goal configuration $\mathbf{p}^*:=(\mathbf{p}_1^*, \mathbf{p}_2^*, \ldots, \mathbf{p}_n^*) \in \mathcal{W}^n$. We assume that both the initial configuration and the target configuration of the objects are known. In addition to the known boundary of the workspace $\partial \mathcal{W}$, the workspace is cluttered by an unknown number of fixed, disjoint, convex obstacles of unknown position and size, denoted by $\mathcal{O}:=(O_1, O_2, \ldots)$. To simplify the notation, also define $\mathcal{O}_w := \mathcal{O} \cup \partial \mathcal{W}$.

We adopt the following assumptions to guarantee that any robot-object pair can go around any obstacle in the workspace along any possible direction, introduced only to facilitate the proofs of our formal results, without being necessary for the existence of some solution to the problem. 

\begin{assumption}[Obstacle separation] \label{assumption:obstacleseparation}
The obstacles $\mathcal{O}$ in the workspace are separated from each other by clearance\footnote{Here the clearance between two sets $A$ and $B$ is defined as $d(A,B):=\inf\{\|\mathbf{a}-\mathbf{b}\| \, | \, \mathbf{a} \in A, \mathbf{b} \in B\}$} of at least $d(O_i,O_j) > 2(r + \max_k \rho_k), i \ne j$, with $k$ an index spanning the set of movable objects. They are also separated from the boundary of the (potentially nonconvex) workspace $\mathcal{W}$ by at least $d(O_i,\partial \mathcal{W}) > 2(r + \max_k \rho_k)$ for all $i$. \label{assumption:obstacle_separation}
\end{assumption}

Assumption \ref{assumption:obstacleseparation} means that there exists $\eta \in \mathbb{R}_{>0}$ such that 
\begin{equation}
\eta = \min \left\{ \min_{\substack{i,j \\ i \neq j}}d(O_i,O_j),\min\limits_i d(O_i,\partial \mathcal{W}) \right \} \label{eq:eta}
\end{equation}
and $\eta > 2(r + \max_k \rho_k)$.

Also, in order to ensure successful positioning of all the objects to their target configuration using reactive control schemes, it is convenient to impose a further constraint on how tightly packed the desired goal configuration can be. 

\begin{assumption}[Admissible object goals]
For any object $i \in \{1,\ldots, n\}$, $d(\mathbf{p}_i^*,\mathcal{O}_w) > \rho_i + 2r$. \label{assumption:admissible_goals}
\end{assumption}

The robot's gripper can either be engaged or disengaged; we will write $g=1$ when the gripper is engaged and $g=0$ when it is disengaged.

\begin{figure}[t]
\centering
\includegraphics[width=0.7\columnwidth]{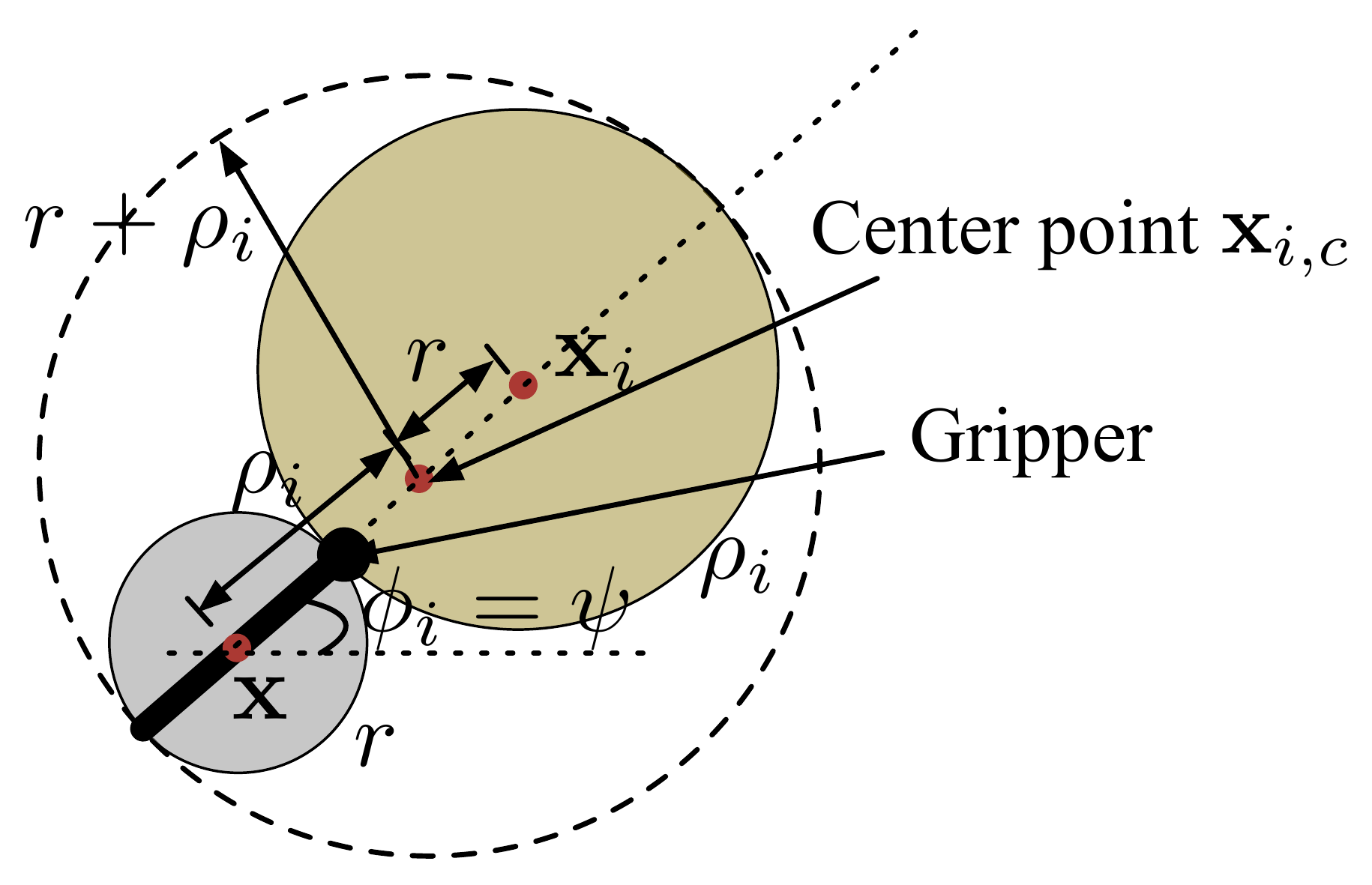}
\caption{A depiction of a disk-shaped robot with radius $r$ (grey) moving a disk-shaped object with radius $\rho_i$ (yellow).}
\label{fig:gripping_contact}
\end{figure}

In order to accomplish the task of bringing every object to its designated goal position, we endow the deliberative planner with a set of three symbolic output action commands: 
\begin{itemize}
\item $\textsc{MoveToObject}(i,\mathcal{P})$ instructing the robot to move and grasp the object $i$ along the piecewise continuously differentiable path $\mathcal{P}:[0,1] \rightarrow \mathcal{W}$ such that $\mathcal{P}(0) = \mathbf{x}$ and $\mathcal{P}(1) = \mathbf{p}_i$.
\item $\textsc{PositionObject}(i,\mathcal{P})$ instructing the robot to push the (assumed already grasped) object $i$ toward its designated goal position, $\mathbf{p}_i^*$, along the piecewise continuously differentiable path $\mathcal{P}:[0,1] \rightarrow \mathcal{W}$ such that $\mathcal{P}(0) = \mathbf{p}_i$ and $\mathcal{P}(1) = \mathbf{p}_i^*$.
\item $\textsc{Move}(\mathcal{P})$ instructing the robot to move along the piecewise continuously differentiable path $\mathcal{P}:[0,1] \rightarrow \mathcal{W}$ such that $\mathcal{P}(0) = \mathbf{x}$.
\end{itemize}
This symbolic command set, comprising the interface between the deliberative and reactive components of our planner enforces the following problem decomposition into the complementary pair:
\begin{enumerate}
\item Find a \textit{symbolic plan}, i.e a sequence of symbolic actions whose successful implementation is guaranteed to complete the task.
\item Implement each of the symbolic actions using the appropriate commands $\mathbf{u}_{ku}$ according to the robot's equations of motion shown in \eqref{eq:robotEOM}, while avoiding the perceived unanticipated by the deliberative planner obstacles.
\end{enumerate} 

\begin{figure}[t]
\centering
\includegraphics[width=1.0\columnwidth]{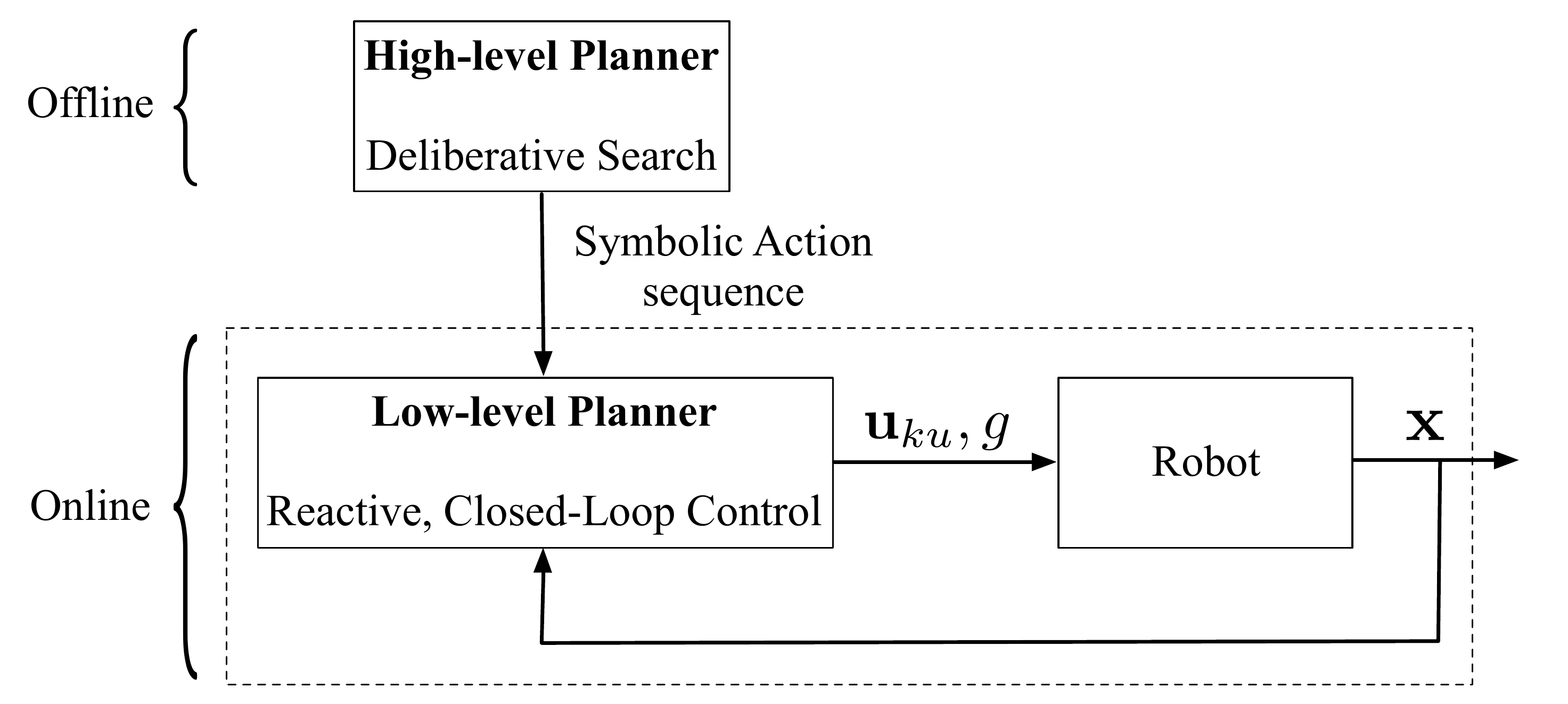}
\caption{An outline of the control approach followed in order to position the objects. A high-level, deliberative planner outputs a sequence of symbolic actions that are realized and executed sequentially in low-level using a reactive controller.}
\label{fig:control_idea}
\end{figure}

Fig. \ref{fig:control_idea} depicts this problem decomposition and the associated interface between the deliberative and reactive components of our architecture.

\section{DELIBERATIVE PLANNER}
\label{sec:angelic}

In order to obtain plans suitable for the reactive planner to track, we use a high-level planner that combines the factored orbital random geometric graph (FORGG) construction \cite{WAFR16WRVB} with the approximate angelic A* (AAA*) search algorithm \cite{vegabrown2017optimal}. 
FORGG extends the asymptotic optimality guarantees of the PRM* algorithm to problems involving discontinuous differential constraints like contact and object manipulation. 
Searching this planning graph using conventional methods like A* is computationally expensive, due to the size of search space. 
To facilitate efficient search, we employ the angelic semantics developed by Marthi~\emph{et al}.~\cite{marthi2008angelic} to encode bounds on the possible cost of sets of possible plans.
AAA* uses these bounds to guide the search, allowing large parts of the search space to be pruned away and accelerating the search for a near-optimal
high-level plan.

With this construction, the deliberative planner is supplied with the initial position and size of the robot and the objects to be placed, along with any information assumed to be known (boundary of the workspace, walls, interior obstacles etc.) and outputs a series of symbolic action commands ($\textsc{MoveToObject}$, $\textsc{PositionObject}$, $\textsc{Move}$) each associated with a collision-free path $\mathcal{P}$ in order to accomplish the task at hand.
\section{REACTIVE PLANNING FOR SINGLE ROBOTS}
\label{sec:reactive}
In this section we describe the (low-level) reactive algorithms which guarantee collision avoidance and (almost) global convergence\footnote{It is well-known that the basin of a point attractor in a non-contractible space must exclude a set of measure zero \cite{koditschek-aam-1990}.} to the plan provided by the (high-level) deliberative planner, described in Section \ref{sec:angelic}. First, we focus on the navigation problem of a single (fully actuated or nonholonomically-constrained) robot, using tools from \cite{arslan_kod_ICRA2016B} and \cite{arslan_kod_WAFR2016}, and we will show in Section \ref{sec:models} how to extend these principles for the case of gripping contact.

\subsection{Doubly-reactive planner for holonomic robots}
\label{subsec:reactiveplanner_holonomic}
First we consider a fully actuated disk-shaped robot centered at $\mathbf{x} \in \mathbb{R}^2$ with radius $r>0$, moving in a closed-convex environment (denoted by $\mathcal{W} \subset \mathbb{R}^2$) towards a goal location $\mathbf{x}^* \in \mathbb{R}^2$. Although we use a differential drive robot for our assembly problem here, we find it useful to present the basic algorithm for fully actuated robots, especially since it will be used in Section \ref{sec:models}. The robot dynamics are assumed to be described by
\begin{equation}
\dot{\mathbf{x}} = \mathbf{u}(\mathbf{x})
\end{equation}
with $\mathbf{u} \in \mathbb{R}^2$ the input. The sensory measurement of the LIDAR at $\mathbf{x} \in \mathcal{W}$ is modeled as in \cite{arslan_kod_WAFR2016} by a polar curve $\rho_\mathbf{x} : (-\pi,\pi] \rightarrow [0,R]$ as follows\footnote{See \cite{arslan_kod_WAFR2016} for a discussion on the choice of LIDAR range $R$ to avoid obstacle occlusions.}
\begin{equation}
\rho_\mathbf{x}(\theta) := \min \left( \begin{matrix}
R \\ \min \left\{ \|\mathbf{p}-\mathbf{x} \| \, | \, \mathbf{p} \in \partial \mathcal{W}, \text{atan2}(\mathbf{p}-\mathbf{x}) = \theta \right\} \\ \min\limits_{i} \left\{ \|\mathbf{p}-\mathbf{x} \| \, | \, \mathbf{p} \in O_i, \text{atan2}(\mathbf{p}-\mathbf{x}) = \theta \right\}
\end{matrix}\right)
\end{equation}
We will also use the definitions of free space $\mathcal{F}$, line-of-sight local workspace $\mathcal{LW_L}(\mathbf{x})$ and line-of-sight local free space $\mathcal{LF_L}(\mathbf{x})$ at $\mathbf{x}$ from \cite{arslan_kod_WAFR2016}.

Under the preceding definitions, it is shown in \cite{arslan_kod_WAFR2016} that the control law
\begin{equation}
\mathbf{u}(\mathbf{x}) = -k \left(\mathbf{x}-\Pi_{\mathcal{LF_L}(\mathbf{x})}(\mathbf{x}^*) \right), k \in \mathbb{R} \label{eq:law_holonomic}
\end{equation}
with $\Pi_A:\mathbb{R}^2 \rightarrow A$ denoting the projection function onto a convex subset $A \subseteq \mathbb{R}^2$, i.e
\begin{equation}
\Pi_A(\mathbf{q}) := \argmin\limits_{\mathbf{a} \in A} \|\mathbf{a}-\mathbf{q}\|
\end{equation}
asymptotically drives almost all configurations in $\mathcal{F}$ to the goal $\mathbf{x}^*$ while avoiding obstacles and not increasing the Euclidean distance to the goal along the way.

\subsection{Reactive path following}
\label{subsec:reactivepathfollowing}
For a fixed goal $\mathbf{x}^*$, the reactive control law in \eqref{eq:law_holonomic} guarantees convergence only for convex workspaces (punctured by obstacles).

Therefore, inspired by \cite{arslan_kod_ICRA2017}, we apply the idea from Section \ref{subsec:reactiveplanner_holonomic} to the problem of a robot following a navigation path $\mathcal{P}:[0,1] \rightarrow \accentset{\circ}{\mathcal{F}}$, that joins a pair of initial and final configurations $\mathbf{x}^0, \mathbf{x}^1 \in \accentset{\circ}{\mathcal{F}}$ in a potentially nonconvex workspace and lies in the interior of the free space, i.e $\mathcal{P}(0) = \mathbf{x}^0, \mathcal{P}(1) = \mathbf{x}^1$ and $\mathcal{P}(\alpha) \in \accentset{\circ}{\mathcal{F}}, \forall \alpha \in [0,1]$.

As demonstrated in \cite{arslan_kod_ICRA2017}, the \textit{projected-path goal} $\mathcal{P}(\alpha^*)$ with $\alpha^*$ determined as\footnote{Here $B(\mathbf{q},t):=\{\mathbf{p} \in \mathcal{W} \, | \, ||\mathbf{p}-\mathbf{q}|| \leq t\}$, i.e the ball of radius $t$ centered at $\mathbf{q}$.}
\begin{equation}
\alpha^* = \max\{\alpha \in [0,1] \, | \, \mathcal{P}(\alpha) \in B\left(\mathbf{x},d(\mathbf{x},\partial \mathcal{F}) \right) \} \label{eq:maxalpha}
\end{equation}
replaces $\mathbf{x}^*$ in \eqref{eq:law_holonomic} as the target goal position and is constantly updated as the agent moves along the path. Note that in the LIDAR-based setting presented here, the distance of the agent from the boundary of the free space $d(\mathbf{x},\partial \mathcal{F})$ can easily be determined as
\begin{equation}
d(\mathbf{x},\partial \mathcal{F}) = \min\limits_{\theta} \rho_\mathbf{x}(\theta) - r \label{eq:distancefromboundary}
\end{equation}

\subsection{Reactive wall following}
\label{subsec:reactivewallfollowing}
As described in Section \ref{sec:problemformulation} and shown in Fig.~\ref{fig:environment}, the path $\mathcal{P}$ might not lie in the free space since the deliberative planner is only aware of the boundary of the workspace and not of the position or size of the internal obstacles. For this reason, we present here a novel control law for reactive wall following, inspired from the ``bug algorithm'' \cite{choset_lynch_hutchinson_kantor_burgard_kavraki_thrun_2005}, that exhibits desired formal guarantees.

The wall following law is triggered by saving the current index $\alpha_s^*$ of the path $\mathcal{P}$ when the distance of the agent from the boundary of its free space, given in \eqref{eq:distancefromboundary}, drops below a small critical value $\epsilon$, i.e when $d(\mathbf{x},\partial \mathcal{F}) < \epsilon$. This would imply that the robot enters a ``danger zone'' within the vicinity of an unexpected obstacle. The goal now would be to follow the boundary of that obstacle without losing it, in order to find the path again.

Therefore, the robot first needs to select a specific direction to consistently follow the boundary of the obstacle along that direction. Since our problem is planar, there are only two possible direction choices: clockwise (CW) or counterclockwise (CCW). Also, since the robot has only local information about the obstacle based on the current LIDAR readings, a greedy selection of the wall following direction is necessary.

\begin{figure}[t]
\centering
\includegraphics[width=0.7\columnwidth]{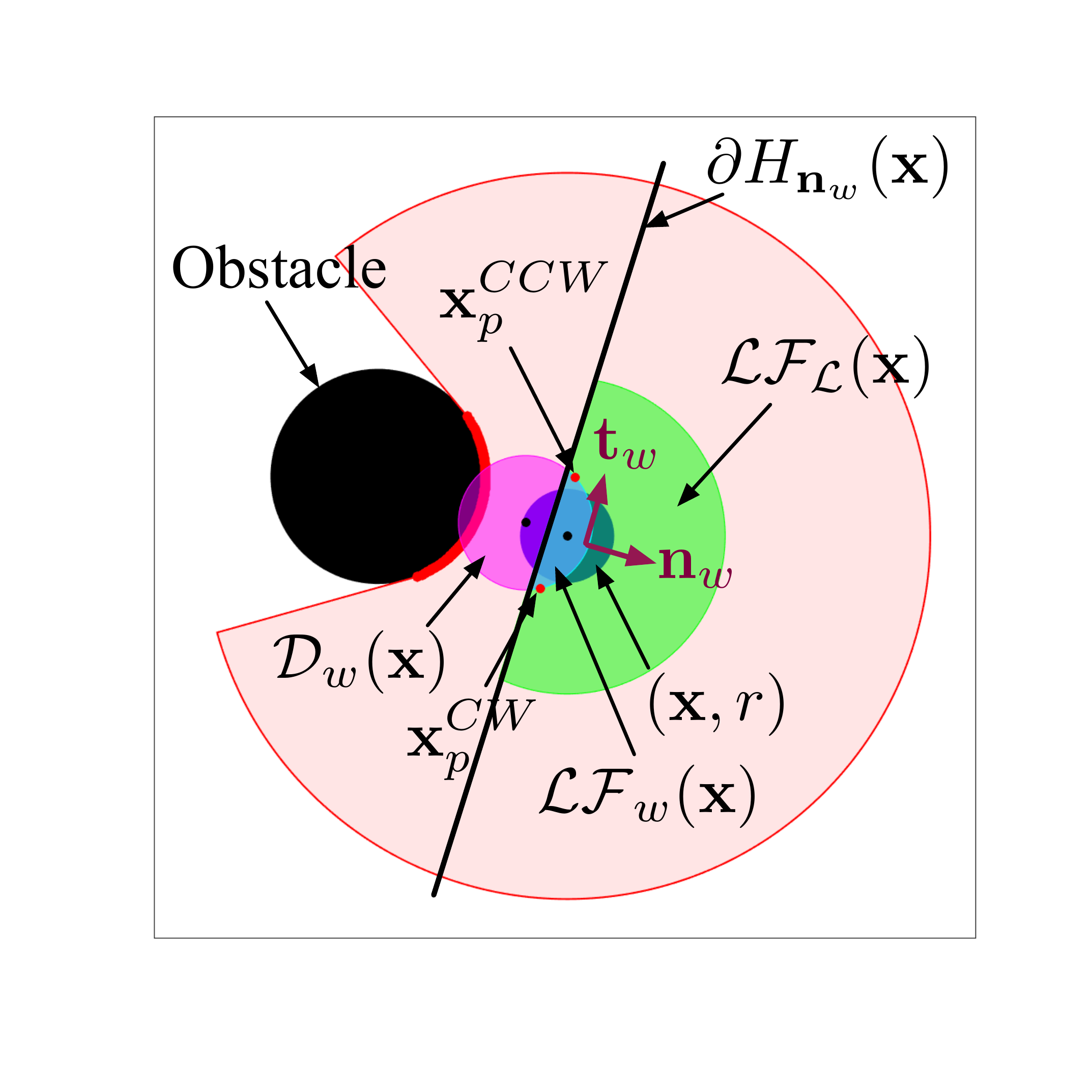}
\caption{An example of computing the wall following local free space $\mathcal{LF}_w(\mathbf{x})$ (cyan) as the intersection of the local free space $\mathcal{LF_L}(\mathbf{x})$ (green) and the offset disk $\mathcal{D}_w$ (magenta) for a robot with radius $r$ positioned at $\mathbf{x}$, encountering an obstacle within its LIDAR footprint $L_{ft}(\mathbf{x})$ (red).}
\label{fig:wallfollowing}
\end{figure}

Let $\theta_m \in (-\pi,\pi]$ be the LIDAR angle such that $\rho_\mathbf{x}(\theta_m) = \min\limits_{\theta} \rho_\mathbf{x}(\theta)$ corresponds to the minimum distance from the blocking obstacle. Let $\mathbf{n}_w(\mathbf{x}) := -(\cos\theta_m,\sin\theta_m)$ denote the normal vector to the boundary of the obstacle at the point of minimum distance and $\mathbf{t}_w(\mathbf{x}) = \mathbf{J} \, \mathbf{n}_w(\mathbf{x})$, with $\mathbf{J} := \begin{bmatrix} 0 & -1 \\ 1 & 0 \end{bmatrix}$, the corresponding tangent vector.

Our proposed method uses the inner product $\mathbf{t}_{w,0} \cdot \mathbf{t}_\mathcal{P}(\alpha_s^*)$, with $\mathbf{t}_{w,0}$ denoting the tangent vector to the boundary of the obstacle at the beginning of the wall following phase and $\mathbf{t}_\mathcal{P}(\alpha_s^*)$ the tangent vector of the path $\mathcal{P}$ at $\alpha_s^*$. Then, the value of a variable $a$ is set to 1 for CCW motion and to -1 for CW motion (fixed for all future time) according to
\begin{equation}
a= \left\{
\begin{array}{rl}
      1, & \text{if} \quad \mathbf{t}_{w,0} \cdot \mathbf{t}_\mathcal{P}(\alpha_s^*) \geq 0\\
      -1, & \text{if} \quad \mathbf{t}_{w,0} \cdot \mathbf{t}_\mathcal{P}(\alpha_s^*) < 0\\
\end{array} 
\right. \label{eq:wallfollowingdirection}
\end{equation}
since $\mathbf{t}_w(\mathbf{x})$ has counterclockwise direction around the obstacle by construction.

Define the \textit{offset disk at $\mathbf{x}$}
\begin{equation}
\mathcal{D}_w(\mathbf{x}) := \left\{\mathbf{p} \in \mathcal{W} \, | \, ||\mathbf{p}-\mathbf{x}_\mathrm{offset}(\mathbf{x})|| \leq \epsilon \right\} \label{eq:offsetdisk}
\end{equation}
with $\epsilon$ selected according to Assumption \ref{assumption:obstacle_separation} to satisfy
\begin{equation}
0 < \epsilon < \frac{1}{2}\left[\eta-2(r + \max_j \rho_j) \right] \label{eq:epsilonchoice}
\end{equation}
with $\eta$ given in \eqref{eq:eta} and
\begin{equation}
\mathbf{x}_\mathrm{offset}(\mathbf{x}) := \mathbf{x}-\left(\rho_\mathbf{x}(\theta_m)-r\right) \mathbf{n}_w(\mathbf{x})
\end{equation}
Then define the \textit{wall following local free space} $\mathcal{LF}_w(\mathbf{x})$ as \textit{at} $\mathbf{x}$
\begin{equation}
\mathcal{LF}_w(\mathbf{x}) := \mathcal{LF_L}(\mathbf{x}) \cap \mathcal{D}_w(\mathbf{x})
\end{equation}
Since $\mathcal{LF_L}(\mathbf{x})$ is convex \cite{arslan_kod_WAFR2016}$, \mathcal{LF}_w(\mathbf{x})$ is convex as the intersection of convex sets.

The \textit{wall following law} is then given as
\begin{equation}
\mathbf{u}(\mathbf{x}) = -k\left(\mathbf{x}-\mathbf{x}_p(\mathbf{x})\right) \label{eq:wallfollowinglaw}
\end{equation}
with
\begin{equation}
\mathbf{x}_p(\mathbf{x}) := \mathbf{x}_\mathrm{offset}(\mathbf{x})+\frac{\epsilon}{2}\mathbf{n}_w(\mathbf{x})+a\frac{\epsilon \sqrt{3}}{2}\mathbf{t}_w(\mathbf{x}) \label{eq:wallfollowinggoal}
\end{equation}

\begin{lemma} \label{lemma:wallfollowing}
If $\mathbf{x} \in \mathcal{F}$ and $d(\mathbf{x},\partial \mathcal{F})<\epsilon$ with $\epsilon$ chosen according to \eqref{eq:epsilonchoice}:
\begin{enumerate}[(i)]
\item The wall following free space $\mathcal{LF}_w(\mathbf{x})$ contains $\mathbf{x}$ in its interior.
\item $\mathcal{LF}_w(\mathbf{x}) = \mathcal{D}_w(\mathbf{x}) \cap H_{\mathbf{n}_w}(\mathbf{x})$ with $H_{\mathbf{n}_w}(\mathbf{x})$ the half space
\begin{equation}
H_{\mathbf{n}_w}(\mathbf{x}) = \left\{ \mathbf{p} \in \mathcal{W} \, | \, (\mathbf{p}-\mathbf{x}_h(\mathbf{x})) \cdot \mathbf{n}_w \geq 0\right\} \nonumber
\end{equation}
and $\mathbf{x}_h(\mathbf{x}) = \mathbf{x}-\frac{1}{2}(\rho_\mathbf{x}(\theta_m)-r)$.
\item The point $\mathbf{x}_p(\mathbf{x})$ lies on the boundary of $\mathcal{LF}_w(\mathbf{x})$.
\end{enumerate}
\end{lemma}

\begin{proof}
\begin{enumerate}[(i)]
\item We know that $\mathbf{x}$ lies in the interior of $\mathcal{LF_L}(\mathbf{x})$ by construction. We also see that $||\mathbf{x}-\mathbf{x}_\mathrm{offset}(\mathbf{x})||=|\rho_\mathbf{x}(\theta_m)-r|=\rho_\mathbf{x}(\theta_m)-r=d(\mathbf{x},\partial \mathcal{F})<\epsilon$. Therefore, $\mathbf{x}$ also lies in the interior of $\mathcal{D}_w(\mathbf{x})$ and, hence, in the interior of $\mathcal{LF}_w(\mathbf{x})$.
\item From the construction of $\mathcal{LF_L}(\mathbf{x})$, we know that $d(\mathbf{x},\partial \mathcal{LF_L}(\mathbf{x})) = \frac{1}{2}(\rho_\mathbf{x}(\theta_m)-r)$. Therefore, in fact, $H_{\mathbf{n}_w}(\mathbf{x})$ is one the half spaces whose intersection constructs $\mathcal{LF_L}(\mathbf{x})$. Its generating hyperplane corresponds to the obstacle $O_k \in \mathcal{O}$ of minimum distance from $\mathbf{x}$ and intersects $\mathcal{D}_w(\mathbf{x})$ at two points, as can be shown by simple substitution in \eqref{eq:offsetdisk}. On the other hand, Assumption \ref{assumption:obstacle_separation} and the choice of $\epsilon$ in \eqref{eq:epsilonchoice}, show that this is the only hyperplane that intersects $\mathcal{D}_w(\mathbf{x})$ and belongs to the boundary of $\mathcal{LF_L}(\mathbf{x})$ and this concludes the proof.
\item Since $\mathbf{t}_w(\mathbf{x}) = \mathbf{J} \, \mathbf{n}_w(\mathbf{x})$, it is not hard to verify that $||\mathbf{x}_p(\mathbf{x})-\mathbf{x}_\mathrm{offset}(\mathbf{x})||=\epsilon$, which shows that $\mathbf{x}_p(\mathbf{x})$ lies on the boundary of $\mathcal{D}_w(\mathbf{x})$. Since $(\mathbf{x}_p(\mathbf{x})-\mathbf{x}_h(\mathbf{x})) \cdot \mathbf{n}_w(\mathbf{x}) = \frac{1}{2}\left[\epsilon-(\rho_\mathbf{x}-r)\right]=\frac{1}{2}(\epsilon-d(\mathbf{x},\partial \mathcal{F}))>0$, we get that $\mathbf{x}_p(\mathbf{x})$ belongs to the half space $H_{\mathbf{n}_w}(\mathbf{x})$. Since in (ii) we proved that $\mathcal{LF}_w(\mathbf{x}) = \mathcal{D}_w(\mathbf{x}) \cap H_{\mathbf{n}_w}(\mathbf{x})$, we conclude that $\mathbf{x}_p(\mathbf{x})$ lies on the boundary of $\mathcal{LF}_w(\mathbf{x})$.
\end{enumerate}
\end{proof}

\begin{proposition} \label{proposition:wallfollowing_properties}
With the choice of $\epsilon$ in \eqref{eq:epsilonchoice}, the wall following law in \eqref{eq:wallfollowinglaw} has the following properties:
\begin{enumerate}[(i)]
\item It is piecewise continuously differentiable.
\item It generates a unique continuously differentiable flow, defined for all future time.
\item It has no stationary points.
\item The free space $\mathcal{F}$ is positively invariant under its flow.
\item Moreover, the set $\left\{ \mathbf{p} \in \mathcal{W} \, \Big | \, \frac{\epsilon}{2} < d(\mathbf{p},\partial \mathcal{F}) < \epsilon \right\}$ is positively invariant under its flow.
\end{enumerate}
\end{proposition}
\begin{proof}
\begin{enumerate}[(i)]
\item Let $O_j \in \mathcal{O}$ denote the obstacle which the robot follows. Since $\rho_\mathbf{x}(\theta_m)$ corresponds to the minimum distance of $\mathbf{x}$ from $O_j$, we can write
\begin{align}
\rho_\mathbf{x}(\theta_m) = & \, ||\mathbf{x} - \Pi_{O_j}(\mathbf{x})|| \nonumber \\
\mathbf{n}_w(\mathbf{x}) = & \frac{\mathbf{x} - \Pi_{O_j}(\mathbf{x})}{||\mathbf{x} - \Pi_{O_j}(\mathbf{x})||} \nonumber
\end{align}
Since metric projections onto closed convex sets (such as $O_j$) are known to be piecewise continuously differentiable \cite{Kuntz-1994,shapiro-1988}, we conclude that both $\rho_\mathbf{x}(\theta_m)$ and $\mathbf{n}_w(\mathbf{x})$ are piecewise continuously differentiable functions of $\mathbf{x}$. Now from \eqref{eq:wallfollowinggoal} and since $\mathbf{n}_w(\mathbf{x}) = \mathbf{J} \, \mathbf{t}_w(\mathbf{x})$ we can write
\begin{equation}
\mathbf{u}(\mathbf{x}) = k\left[\left(\frac{\epsilon}{2}+r-\rho_\mathbf{x}(\theta_m) \right) \mathbf{n}_w(\mathbf{x}) + a \frac{\epsilon \sqrt{3}}{2}\mathbf{J} \mathbf{n}_w(\mathbf{x}) \right] \label{eq:wallfollowinglawsimplified}
\end{equation}
Therefore, we conclude that the wall following law $\mathbf{u}(\mathbf{x})$ is piecewise continuously differentiable as a composition of piecewise continuously differentiable functions.
\item Since piecewise continuously differentiable functions are also locally Lipschitz \cite{chaney-1990}, and since locally Lipschitz functions defined on a compact domain are also globally Lipschitz, we conclude that $\mathbf{u}(\mathbf{x})$ is Lipschitz continuous using (i). The existence, uniqueness and continuous differentiability of its flow follow directly from this property.
\item This follows directly from the form of the wall following law in \eqref{eq:wallfollowinglawsimplified}, since the coefficient corresponding to $\mathbf{t}_w(\mathbf{x}) = \mathbf{J} \, \mathbf{n}_w(\mathbf{x})$ can never be zero.
\item From Lemma \ref{lemma:wallfollowing}, we know that for any $\mathbf{x} \in \mathcal{F}$, the wall following local free space $\mathcal{LF}_w(\mathbf{x})$ is a closed convex subset of $\mathcal{F}$, which is collision-free (as a subset of $\mathcal{LF_L}(\mathbf{x})$) and contains both $\mathbf{x}$ and $\mathbf{x}_p(\mathbf{x})$. Hence, $-k(\mathbf{x}-\mathbf{x}_p(\mathbf{x})) \in T_\mathbf{x} \mathcal{F}$ is either interior directed or at worst tangent to the boundary of $\mathcal{F}$ and this concludes the proof. 
\item Similarly, we see that for any $\mathbf{x} \in \mathcal{F}$ satisfying $d(\mathbf{p},\partial \mathcal{F})= \epsilon$, the choice of $\epsilon$ in \eqref{eq:epsilonchoice} implies that there is a unique obstacle $j = \arg \min\limits_i d(\mathbf{x},O_i)$ such that $-k(\mathbf{x}-\mathbf{x}_p(\mathbf{x})) \in T_\mathbf{x} \mathcal{F}$ is interior directed to the set $\left\{ \mathbf{p} \in \mathbb{R}^2 \, | \, d(\mathbf{p}, O_j) < \epsilon \right\}$ and, hence, interior directed to the set $\left\{ \mathbf{p} \in \mathbb{R}^2 \, | \, d(\mathbf{p}, \partial \mathcal{F}) < \epsilon \right\}$. Similar reasoning leads to the conclusion that $\left\{ \mathbf{p} \in \mathbb{R}^2 \, \Big | \, d(\mathbf{p}, \partial \mathcal{F}) > \frac{\epsilon}{2} \right\}$ is positively invariant under the wall following law, since for $d(\mathbf{x},\partial \mathcal{F}) = \rho_\mathbf{x}(\theta_m)-r = \frac{\epsilon}{2}$, \eqref{eq:wallfollowinglawsimplified} gives $\mathbf{u}(\mathbf{x}) \parallel \mathbf{t}_w(\mathbf{x})$ and this concludes the proof.
\end{enumerate}
\end{proof}

We find it useful to include the following definition
\begin{definition}
The \textit{rate of progress} along the boundary of the observed obstacle at $\mathbf{x}$ is defined as
\begin{equation}
\sigma(\mathbf{x}) := \frac{\mathbf{u}(\mathbf{x}) \cdot \mathbf{t}_w(\mathbf{x})}{||\mathbf{u}(\mathbf{x})||}
\end{equation}
\end{definition}

By combining all these results, we arrive at the following Theorem:
\begin{theorem} \label{theorem:wallfollowing}
With a selection of $\epsilon$ as in \eqref{eq:epsilonchoice}, the wall following law in \eqref{eq:wallfollowinglaw} has no stationary points, leaves the robot's free space $\mathcal{F}$ positively invariant under its unique continuously differentiable flow, and steers the robot along the boundary of a unique obstacle in $\mathcal{O}$ in a clockwise or counterclockwise fashion (according to the selection of a in \eqref{eq:wallfollowingdirection}) with a nonzero rate of progress $\sigma$, while maintaining a distance of at most $(r+\epsilon)$ and no less than $\left( r + \frac{\epsilon}{2}\right)$ from it.
\end{theorem}

In order to prove the theorem, we will make use of the following Proposition.
\begin{proposition} \label{proposition:nochangingobstacle}
Let $\mathbf{x}^t$ denote the robot position at time $t$, with $t=0$ corresponding to the beginning of the wall following phase. Suppose that the flow $\mathbf{x}^t$ is continuous, $k = \arg \min\limits_i d(\mathbf{x}^0,O_i)$ with $O_i \in \mathcal{O}$ and $\epsilon$ satisfies \eqref{eq:epsilonchoice}. Then $d(\mathbf{x}^t, \partial \mathcal{F}) < \epsilon$ implies $k = \arg \min\limits_i d(\mathbf{x}^t,O_i)$ for all $t>0$.
\end{proposition}
\begin{proof}
Suppose this is not true. Then there exists $t_c>0$ such that $l = \arg \min\limits_i d(\mathbf{x}^{t_c},O_i) \neq k$ and $d(\mathbf{x}^{t_c}, \partial \mathcal{F}) < \epsilon$.\footnote{Here we slightly abuse the notation, since $O_l$ could as well correspond to the boundary of the workspace $\partial \mathcal{W}$. The analysis still holds, because of the particular bounds provided for $\epsilon$ in \eqref{eq:epsilonchoice}.} This implies that $d(\mathbf{x}^{t_c}, \partial \mathcal{F}) = d(\mathbf{x}^{t_c}, O_l)-r$, which gives $d(\mathbf{x}^{t_c}, O_l) < r + \epsilon < \frac{1}{2}d(O_k,O_l)$ from the choice of $\epsilon$. Hence, from the triangle inequality, we have
\begin{align*}
d(\mathbf{x}^{t_c},O_k) & \geq d(O_k,O_l)-d(\mathbf{x}^{t_c},O_l) \\
& > \frac{1}{2}d(O_k,O_l) \\
& \geq \frac{1}{2}\eta \\
& > (r + \max_j \rho_j) + \epsilon \\
& > r + \epsilon
\end{align*}
from \eqref{eq:epsilonchoice}. Therefore, we get that $d(\mathbf{x}^{t_c},O_k) > r + \epsilon$. From the assumptions, we have $d(\mathbf{x}^0,\partial \mathcal{F}) = d(\mathbf{x}^0,O_k) - r < \epsilon$, which implies $d(\mathbf{x}^0,O_k) < r + \epsilon$. Since $d(\mathbf{x}^t,O_k)$ is continuous, as the composition of the distance function to a subset of $\mathbb{R}^2$ with the continuous flow $\mathbf{x}^t$, we can use the Intermediate Value Theorem to deduce that there exists a time $t_m \in (0,t_c)$ such that $d(\mathbf{x}^{t_m},O_k) = r + \epsilon$. From the choice of $\epsilon$ in \eqref{eq:epsilonchoice}, this implies that
\begin{equation}
d(\mathbf{x}^{t_m},O_k) < \frac{1}{2} \eta
\end{equation}
so that $d(\mathbf{x}^{t_m},\partial \mathcal{F}) = d(\mathbf{x}^{t_m},O_k) - r = \epsilon$, violating the assumption that $d(\mathbf{x}^t, \partial \mathcal{F}) < \epsilon$ for all $t>0$ and leading to a contradiction.
\end{proof}

\begin{proof}[Proof of Theorem \ref{theorem:wallfollowing}]
The fact that, under the wall following law in \eqref{eq:wallfollowinglaw}, the robot follows the boundary of a unique obstacle follows readily from Proposition \ref{proposition:nochangingobstacle}, the continuity of the flow and the positive invariance of $\left\{ \mathbf{p} \in \mathbb{R}^2 \, | \, d(\mathbf{p},\partial \mathcal{F}) < \epsilon \right\}$ as derived in Proposition \ref{proposition:wallfollowing_properties}. Finally, from \eqref{eq:wallfollowinglawsimplified}, notice that
\begin{equation}
\mathbf{u}(\mathbf{x}) \cdot \mathbf{t}_w(\mathbf{x}) = a\frac{\epsilon \sqrt{3}}{2}
\end{equation}
which implies that $|\sigma(\mathbf{x})| \in (0,1]$ and $\text{sign}(\sigma(\mathbf{x})) = \text{sign}(a)$. Since $\mathbf{u}(\mathbf{x}) \cdot \mathbf{t}_w(\mathbf{x})$ expresses the component of $\mathbf{u}(\mathbf{x})$ along the tangent to the obstacle boundary $\mathbf{t}_w(\mathbf{x})$, is always nonzero and does not change sign, we conclude that the robot will follow the boundary of the obstacle clockwise or counterclockwise, depending on $a$. The rest of the claims derive immediately from Proposition \ref{proposition:wallfollowing_properties}.
\end{proof}

The robot exits the wall following mode and returns to the path following mode once it encounters the path again, i.e when $\alpha^*=\max\{\alpha \in [0,1] \, | \, \mathcal{P}(\alpha) \in B\left(\mathbf{x},d(\mathbf{x},\partial \mathcal{F}) \right) \} > \alpha_s^*$. An immediate Corollary of Theorem \ref{theorem:wallfollowing}, along with path continuity of $\mathcal{P}$ and Assumptions \ref{assumption:obstacle_separation} and \ref{assumption:admissible_goals} is the following:
\begin{corollary}
If the robot enters the wall following mode, it will exit it in finite time and return to the path following mode.
\end{corollary}

Finally, since both the path following law \cite{arslan_kod_ICRA2017} and the wall following law generate continuously differentiable flows, we find it useful to explicitly state the following result, in the sense of sequential composition \cite{burridge-ijrr-1999}.
\begin{theorem} \label{theorem:flow}
In a workspace where Assumption \ref{assumption:obstacle_separation} is satisfied, any composition of path following and wall following phases generates a unique piecewise continuously differentiable flow for $\mathbf{x}$, defined for all future time.
\end{theorem}

\subsection{Extension to nonholonomic robots}
\label{subsec:reactiveplanner_nonholonomic}
As shown in \cite{arslan_kod_WAFR2016}, the preceding results can easily be extended for the case of a differential-drive robot driving towards a goal $\mathbf{x}^*$, whose dynamics are given in \eqref{eq:robotEOM}. Here, we will use a slightly different than \cite{arslan_kod_WAFR2016} control law since the robot possesses a gripper and must only move in the forward direction to grasp objects. The following inputs are used
\begin{align}
v & = \max \left\{-k \begin{bmatrix}
\cos\psi \\ \sin\psi
\end{bmatrix}^T \, \left(\mathbf{x}-\Pi_{\mathcal{LF}_v(\mathbf{x})}(\mathbf{x}^*) \right),0 \right \}\label{eq:law_nonholonomic_v} \\
\omega & = -k \, \text{atan2} \left( \beta_2,\beta_1 \right) \label{eq:law_nonholonomic_omega}
\end{align}
with
\begin{align}
\beta_1 & = \begin{bmatrix}
\cos\psi \\ \sin\psi
\end{bmatrix}^T \left(\mathbf{x}-\frac{\Pi_{\mathcal{LF}_\omega(\mathbf{x})}(\mathbf{x}^*) + \Pi_{\mathcal{LF_L}(\mathbf{x})}(\mathbf{x}^*)}{2}\right)\\
\beta_2 & = \begin{bmatrix}
-\sin\psi \\ \cos\psi
\end{bmatrix}^T \left(\mathbf{x}-\frac{\Pi_{\mathcal{LF}_\omega(\mathbf{x})}(\mathbf{x}^*) + \Pi_{\mathcal{LF_L}(\mathbf{x})}(\mathbf{x}^*)}{2}\right)
\end{align}
in order to constrain the robot motion to forward only and align with the desired target respectively. Here $\mathcal{LF}_v(\mathbf{x}), \mathcal{LF}_\omega(\mathbf{x})$ are used as in \cite{arslan_kod_WAFR2016}.

Based on the preceding analysis, for a differential drive robot, we will use $\mathbf{x}^* = \mathcal{P}(\alpha^*)$ (with $\alpha^*$ shown in \eqref{eq:maxalpha}) in the path following mode and $\mathbf{x}^* = \mathbf{x}_p(\mathbf{x})$ in the wall following mode. The following Theorem summarizes the qualitative properties of the wall following law for differential drive robots.
\begin{theorem} \label{theorem:wallfollowingunicycle}
With a selection of $\epsilon$ as in \eqref{eq:epsilonchoice}, the unicycle wall following law in \eqref{eq:law_nonholonomic_v}, \eqref{eq:law_nonholonomic_omega} with $\mathbf{x}^* = \mathbf{x}_p(\mathbf{x})$ as in \eqref{eq:wallfollowinggoal} leaves the robot's free space $\mathcal{F}$ positively invariant under its unique continuously differentiable flow, aligns the robot with $a \, \mathbf{t}_w(\mathbf{x})$ (according to the selection of a in \eqref{eq:wallfollowingdirection}) in finite time and steers the robot along the boundary of a unique obstacle in $\mathcal{O}$ in a clockwise or counterclockwise fashion (depending on $a$) with a nonzero rate of progress $\sigma$ afterwards, while maintaining a distance of at most $(r+\epsilon)$ from it.
\end{theorem}
\begin{proof}[Proof sketch]
Positive invariance of $\mathcal{F}$ is guaranteed from \cite{arslan_kod_WAFR2016} with the particular choice of $\mathcal{LF}_v(\mathbf{x}), \mathcal{LF}_\omega(\mathbf{x})$. The existence, uniqueness and continuous differentiability of the flow are guaranteed through the piecewise continuous differentiability of the vector field, similarly to the proof of Proposition \ref{proposition:wallfollowing_properties}-(ii). We can also prove the positive invariance of the set $\{ \mathbf{p} \in \mathcal{W} \, | \, d(\mathbf{p}, \partial \mathcal{F}) < \epsilon\}$ by the particular selection of $\mathbf{x}_p(\mathbf{x})$, as in Proposition \ref{proposition:wallfollowing_properties}-(v). The only problem, unique to differential drive robots, is that the robot orientation might not originally be aligned with $\mathbf{x}_p(\mathbf{x})$. However, since the robot is not allowed to move backwards (from \eqref{eq:law_nonholonomic_v}), the angular control law in \eqref{eq:law_nonholonomic_omega} with $\mathbf{x}^* = \mathbf{x}_p(\mathbf{x})$ will force the robot to turn towards $\mathbf{x}_p(\mathbf{x})$ in finite time and continue following that direction onwards.
\end{proof}
We summarize the proposed method for switching between a path following and a wall following phase and generating velocity commands for a differential drive robot following a reference path $\mathcal{P}$ in Algorithm \ref{algorithm:velocity_commands_nonholonomic}, with the definition of an auxiliary symbolic action $\textsc{NavigateRobot}(\mathcal{P},r,\epsilon,\delta)$.

\begin{algorithm}
\begin{algorithmic}
\Function{NavigateRobot}{$\mathcal{P},r,\epsilon,\delta$}
\State $\texttt{mode} \gets \texttt{PathFollowing}$ \Comment{Initialize \texttt{mode}}
\Do
\State $\mathbf{x} \gets \text{Read Robot State}$
\State $\psi \gets \text{Read Robot Orientation}$
\State $\rho_\mathbf{x} \gets \text{Read LIDAR}$
\State $d \gets \min\limits_{\theta} \rho_\mathbf{x}(\theta) - r$
\State $\alpha^* \gets \text{Find maximum path index} $ \Comment{\eqref{eq:maxalpha}}
\If{$\texttt{mode}=\texttt{PathFollowing}$}
\State $\mathcal{LF_L}(\mathbf{x}) \gets \text{Find local free space}$ \Comment{\cite[(28)]{arslan_kod_WAFR2016}}
\State $\mathbf{x}^* \gets \mathcal{P}(\alpha^*)$
\If{$d<\epsilon$}
\State $\texttt{mode} \gets \texttt{WallFollowing}$
\State $\alpha_s^* \gets \alpha^*$
\State $a \gets \text{Find wall following direction}$ \Comment{\eqref{eq:wallfollowingdirection}}
\EndIf
\ElsIf{$\texttt{mode}=\texttt{WallFollowing}$}
\State $\theta_m \gets \arg \min\limits_\theta \rho_\mathbf{x}(\theta)$
\State $\mathbf{n}_w \gets -(\cos\theta_m,\sin\theta_m)$
\State $\mathbf{t}_w \gets (\sin\theta_m,-\cos\theta_m)$
\State $\mathbf{x}_\mathrm{offset} \gets \mathbf{x}-\left(\rho_\mathbf{x}(\theta_m)-r\right) \mathbf{n}_w$
\State $\mathbf{x}_p \gets \mathbf{x}_\mathrm{offset}+\frac{\epsilon}{2}\mathbf{n}_w+a\frac{\epsilon \sqrt{3}}{2}\mathbf{t}_w$
\State $\mathbf{x}^* \gets \mathbf{x}_p$
\If{$\alpha^* > \alpha_s^*$}
\State $\texttt{mode} \gets \texttt{PathFollowing}$
\EndIf
\EndIf
\State $v \gets \text{Find Linear Velocity command}$ \Comment{\eqref{eq:law_nonholonomic_v}}
\State $\omega \gets \text{Find Angular Velocity command}$ \Comment{\eqref{eq:law_nonholonomic_omega}}
\State $\mathbf{u}_{ku} \gets (v,\omega)$
\State $\textbf{COMMAND} \quad \mathbf{u}_{ku}$
\doWhile{$||\mathbf{x}-\mathcal{P}(1)||>\delta$}
\State $\textbf{return}$
\EndFunction
\end{algorithmic}
\caption{Generating velocity commands for a nonholonomic robot with radius $r$ following a reference path $\mathcal{P}$.} \label{algorithm:velocity_commands_nonholonomic}
\end{algorithm}

\section{REACTIVE PLANNING FOR GRIPPING CONTACT}
\label{sec:models}

In this Section, we describe a method for generating suitable motion commands online for two objects in contact, of which one is a differential drive robot and uses a gripper to push the other, passive object on the plane. Our method consists of generating ``virtual'' commands for different points of interest in the robot-object pair and translating them to ``actual'' commands for the robot using simple kinematic maps.

\subsection{Gripping contact kinematics}
\label{subsec:grippingcontactkinematics}

Consider the robot gripping an object $i$, as shown in Fig. \ref{fig:gripping_contact}. We can find the position of the object center of mass $\mathbf{x}_i \in \mathcal{W} $ from the position of the robot center of mass $\mathbf{x}$ as
\begin{equation}
\mathbf{x}_i := \mathbf{x} + (\rho_i+r) \, \mathbf{e}_\parallel \label{eq:objectposition}
\end{equation}
where $\phi_i = \text{atan2}(\mathbf{x}_i-\mathbf{x})$ and $\mathbf{e}_\parallel = (\cos\phi_i,\sin\phi_i) \in \mathbb{R}^2$ is the unit vector along the line connecting the two bodies. Since, the velocity of the object center of mass will be $\dot{\mathbf{x}}_i = \dot{\mathbf{x}} + (\rho_i+r) \, \dot{\phi}_i \, \mathbf{e}_\perp$ with $\mathbf{e}_\perp = (-\sin\phi_i,\cos\phi_i) \perp \mathbf{e}_\parallel$, and since the robot has a grip on the object along its line of motion, so that the orientation of the robot $\psi$ is always equal to the robot-object bearing angle $\phi_i$, we can use \eqref{eq:robotEOM} to write
\begin{equation}
\dot{\mathbf{x}}_i = \mathbf{T}_i \, \mathbf{u}_{ku} \label{eq:nonholonomictransformation}
\end{equation}
with the Jacobian $\mathbf{T}_i$ given by
\begin{equation}
\mathbf{T}_i = \begin{bmatrix}
\cos\psi & -(\rho_i+r)\sin\psi \\ \sin\psi & (\rho_i+r)\cos\psi
\end{bmatrix}
\end{equation}
and $\mathbf{u}_{ku}=(v,\omega)$ the input vector as defined above.

Similarly, consider the circumscribed circle enclosing the robot and the object with radius $(\rho_i + r)$, as shown in Fig. \ref{fig:gripping_contact}. Its \textit{center point} is located at
\begin{equation}
\mathbf{x}_{i,c} = \mathbf{x} + \rho_i \, \mathbf{e}_\parallel \label{eq:centerposition}
\end{equation}
Following a similar procedure as above, we can show that 
\begin{equation}
\dot{\mathbf{x}}_{i,c} = \mathbf{T}_{i,c} \, \mathbf{u}_{ku} \label{eq:nonholonomictransformationcenter}
\end{equation}
with the Jacobian $\mathbf{T}_{i,c}$ given by
\begin{equation}
\mathbf{T}_{i,c} = \begin{bmatrix}
\cos\psi & -\rho_i\sin\psi \\ \sin\psi & \rho_i\cos\psi
\end{bmatrix}
\end{equation}

\subsection{Generating virtual commands}
For the planning process, the fact that both $\mathbf{T}_i$ and $\mathbf{T}_{i,c}$ are always non-singular implies that we can describe the robot-object pair as either a dynamical system of the form
\begin{equation}
\dot{\mathbf{x}}_i = \mathbf{u}_i(\mathbf{x}_i) \label{eq:planningobject}
\end{equation}
or a dynamical system of the form
\begin{equation}
\dot{\mathbf{x}}_{i,c} = \mathbf{u}_{i,c}(\mathbf{x}_{i,c}) \label{eq:planningcenter}
\end{equation}
since we can always prescribe (virtual) arbitrary velocity commands $\mathbf{u}_i$ or $\mathbf{u}_{i,c}$ for either the object itself or for the center point and then translate them to (actual) inputs $\mathbf{u}_{ku}$ through \eqref{eq:nonholonomictransformation} or \eqref{eq:nonholonomictransformationcenter} respectively ($\mathbf{u}_{ku} = \mathbf{T}_i^{-1} \, \mathbf{u}_i$ or $\mathbf{u}_{ku} = \mathbf{T}_{i,c}^{-1} \, \mathbf{u}_{i,c}$).

Since the circumscribed circle centered at $\mathbf{x}_{i,c}$ is the smallest circle enclosing both the robot and the object and since Assumption \ref{assumption:obstacleseparation} guarantees only that $\eta > 2(r+\max_k \rho_k)$, we conclude that it is beneficial to consider the dynamical system \eqref{eq:planningcenter} (and generate virtual commands for the center point $\mathbf{x}_{i,c}$) when following the path $\mathcal{P}$ that the high-level planner provides. However, this will eventually position $\mathbf{x}_{i,c}$ to $\mathbf{p}_i^*$, instead of the object $\mathbf{x}_i$ (which is desired). Therefore, once the center point is placed to $\mathbf{p}_i^*$, we have to switch to the system \eqref{eq:planningobject} and generate virtual commands for the object $\mathbf{x}_i$ to carefully position it to $\mathbf{p}_i^*$. Assumption \ref{assumption:admissible_goals} guarantees that this is always possible. We can think of generating commands $\mathbf{u}_i$ and $\mathbf{u}_{i,c}$ as a trade-off between careful object positioning and agility in avoiding obstacles respectively.

\subsection{LIDAR Range transformation}

As described above, the robot-object pair is treated as a single holonomic agent with radius $\rho_i+r$ centered at $\mathbf{x}_{i,c}$ when following the reference path $\mathcal{P}$. However, we know that the LIDAR is positioned on the robot and its range measurements are given with respect to $\mathbf{x}$. Therefore, we need a mechanism for translating these measurements from $\mathbf{x}$ to $\mathbf{x}_{i,c}$. To this end, we describe the observed points from the LIDAR using the function $\mathbf{x}_\mathrm{LIDAR}:(-\pi,\pi] \rightarrow \mathcal{W}$
\begin{equation}
\mathbf{x}_\mathrm{LIDAR}(\theta) = \mathbf{x}+\rho_\mathbf{x}(\theta) \, (\cos\theta,\sin\theta)
\end{equation}
and find the equivalent ranges from $\mathbf{x}_{i,c}$ as
\begin{equation}
\rho_{\mathbf{x}_{i,c}}(\theta) = \min \{R-\rho_i,||\mathbf{x}_\mathrm{LIDAR}(\theta)-\mathbf{x}_{i,c}||\} \label{eq:lidarobjecttf}
\end{equation}
since $R-\rho_i$ is the minimum distance that can be observed from $\mathbf{x}_{i,c}$ when no obstacles are present and corresponds to the ray along the orientation $\psi$ of the robot towards the object.

We summarize the proposed algorithm for switching between a path following and a wall following phase and generating velocity commands for a robot-object pair following a reference path $\mathcal{P}$ in Algorithm \ref{algorithm:velocity_commands_robotobject}, with the definition of an auxiliary symbolic action $\textsc{NavigateRobotObject}(\mathcal{P},r,\rho_i,\epsilon,\delta)$.

\begin{algorithm}
\begin{algorithmic}
\Function{NavigateRobotObject}{$\mathcal{P},r,\rho_i,\epsilon,\delta$}
\State $\texttt{mode} \gets \texttt{PathFollowing}$ \Comment{Initialize \texttt{mode}}
\Do
\State $\mathbf{x} \gets \text{Read Robot State}$
\State $\psi \gets \text{Read Robot Orientation}$
\State $\rho_\mathbf{x} \gets \text{Read LIDAR}$
\State $\mathbf{x}_{i,c} \gets \text{Find center of circumscribed circle}$ \Comment{\eqref{eq:centerposition}}
\State $\rho_{\mathbf{x}_{i,c}} \gets \text{Transform LIDAR readings}$ \Comment{\eqref{eq:lidarobjecttf}}
\State $d \gets \min\limits_{\theta} \rho_{\mathbf{x}_{i,c}}(\theta) - (r+\rho_i)$
\State $\alpha^* \gets \text{Find maximum path index} $ \Comment{\eqref{eq:maxalpha}}
\If{$\texttt{mode}=\texttt{PathFollowing}$}
\State $\mathcal{LF_L}(\mathbf{x}_{i,c}) \gets \text{Find local free space}$ \Comment{\cite{arslan_kod_WAFR2016}}
\State $\mathbf{x}_{i,c}^* \gets \Pi_{\mathcal{LF_L}(\mathbf{x}_{i,c})}(\mathcal{P}(\alpha^*))$
\If{$d<\epsilon$}
\State $\texttt{mode} \gets \texttt{WallFollowing}$
\State $\alpha_s^* \gets \alpha^*$
\State $a \gets \text{Find wall following direction}$ \Comment{\eqref{eq:wallfollowingdirection}}
\EndIf
\ElsIf{$\texttt{mode}=\texttt{WallFollowing}$}
\State $\theta_m \gets \arg \min\limits_\theta \rho_{\mathbf{x}_{i,c}}(\theta)$
\State $\mathbf{n}_w \gets -(\cos\theta_m,\sin\theta_m)$
\State $\mathbf{t}_w \gets (\sin\theta_m,-\cos\theta_m)$
\State $\mathbf{x}_\mathrm{offset} \gets \mathbf{x}_{i,c}-\left(\rho_{\mathbf{x}_{i,c}}(\theta_m)-r-\rho_i\right) \mathbf{n}_w$
\State $\mathbf{x}_p \gets \mathbf{x}_\mathrm{offset}+\frac{\epsilon}{2}\mathbf{n}_w+a\frac{\epsilon \sqrt{3}}{2}\mathbf{t}_w$
\State $\mathbf{x}_{i,c}^* \gets \mathbf{x}_p$
\If{$\alpha^* > \alpha_s^*$}
\State $\texttt{mode} \gets \texttt{PathFollowing}$
\EndIf
\EndIf
\State $\mathbf{u}_{i,c} \gets -k(\mathbf{x}_{i,c}-\mathbf{x}_{i,c}^*)$ \Comment{Virtual commands}
\State $\mathbf{u}_{ku} \gets \mathbf{T}_{i,c}^{-1} \, \mathbf{u}_{i,c}$ \Comment{Actual commands}
\State $\textbf{COMMAND} \quad \mathbf{u}_{ku}$
\doWhile{$||\mathbf{x}_{i,c}-\mathcal{P}(1)||>r+\delta$}
\State $\textbf{return}$
\EndFunction
\end{algorithmic}
\caption{Generating velocity commands for a nonholonomic robot with radius $r$ moving an object of radius $\rho_i$ along a reference path $\mathcal{P}$.} \label{algorithm:velocity_commands_robotobject}
\end{algorithm}

\section{LOW-LEVEL IMPLEMENTATION OF SYMBOLIC LANGUAGE}
\label{sec:symboliclanguage}

In this section, we describe the low-level implementation and realization of the three symbolic actions introduced in Section \ref{sec:problemformulation}, used to solve our assembly problem.

\subsection{Action $\textsc{MoveToObject}$}
The low-level implementation of this symbolic action is quite straightforward, since the robot just needs to follow the plan provided by the high-level planner and navigate to a specific object using the auxiliary action $\textsc{NavigateRobot}$. The only caveat is that the robot needs to be aligned with the object it needs to pick up in order to engage the gripper. Since, no continuous law can guarantee both position and orientation convergence for a nonholonomically-constrained, differential drive robot \cite{Brockett-83} and a discontinuous law needs to be introduced, we compute
\begin{equation}
\tilde{\alpha} := \min\{ \alpha \in [0,1] \, | \, \mathcal{P}(\alpha) \in B(\mathbf{p}_i,\rho_i+r)\}
\end{equation}
and ``truncate'' the path to $\mathcal{P}([0,\tilde{\alpha}])$. In this way, the robot will navigate to $\mathcal{P}(\tilde{\alpha})$ (within a $\delta$ tolerance) which satisfies $||\mathcal{P}(\tilde{\alpha}) - \mathbf{p}_i||=\rho_i+r$ as desired. Then, in order to align the robot with the object, the linear command $v$ is set to zero and the angular command is set to
\begin{equation}
\omega = -k(\phi_i-\psi)
\end{equation}
until $\phi_i=\psi$. The low-level implementation is shown in Algorithm \ref{algorithm:movetoobject}.

\begin{algorithm}
\begin{algorithmic}[1]
\Function{MoveToObject}{$i,\mathcal{P}$}
\State $\epsilon \gets \text{Set Wall Following Tolerance}$ \Comment{$\epsilon<\eta$}
\State $\delta \gets \text{Set Placement Tolerance}$
\State $\tilde{\alpha} \gets \min\{ \alpha \in [0,1] \, | \, \mathcal{P}(\alpha) \in B(\mathbf{p}_i,\rho_i+r)\}$
\State $\textsc{NavigateRobot}(\mathcal{P}([0,\tilde{\alpha}]),r,\epsilon,\delta)$
\While{$|\phi_i-\psi|>\delta$}
\State $\mathbf{u}_{ku} \gets \left(0, -k(\phi_i-\psi)\right)$ \Comment{Align with object}
\State $\textbf{COMMAND} \quad \mathbf{u}_{ku}$
\EndWhile
\State $g \gets 1$ \Comment{Engage gripper}
\State $\textbf{return}$
\EndFunction
\end{algorithmic}
\caption{Robot navigation to object $\mathbf{p}_i$ along path $\mathcal{P}$}\label{algorithm:movetoobject}
\end{algorithm}

\subsection{Action $\textsc{PositionObject}$}
From the preceding analysis in Section \ref{sec:models} and using the auxiliary action $\textsc{NavigateRobotObject}$, we can construct the $\textsc{PositionObject}$ algorithm as shown in Algorithm \ref{algorithm:positionobject}. Since the task of $\textsc{NavigateRobotObject}$ is to bring the object close enough to the destination in order to allow careful positioning (allowed by Assumption \ref{assumption:admissible_goals}), a final positioning step is required. To this end, instead of generating virtual commands for the center of the circumscribed circle of the robot-object pair as shown in \eqref{eq:planningcenter}, we generate commands for the center of the object itself, as shown in \eqref{eq:planningobject}, according to the following law
\begin{equation}
\mathbf{u}_i = -k(\mathbf{x}_i-\mathbf{p}_i^*)
\end{equation}
These virtual commands are then translated to actual robot commands according to \eqref{eq:nonholonomictransformation}.

\begin{algorithm}
\begin{algorithmic}[1]
\Function{PositionObject}{$i,\mathcal{P}$}
\State $\epsilon \gets \text{Set Wall Following Tolerance}$ \Comment{$\epsilon<\eta$}
\State $\delta \gets \text{Set Placement Tolerance}$
\State $\textsc{NavigateRobotObject}(\mathcal{P},r,\rho_i,\epsilon,\delta)$
\Do
\State $\mathbf{x} \gets \text{Read Robot State}$
\State $\psi \gets \text{Read Robot Orientation}$
\State $\mathbf{x}_i \gets \text{Find object position}$ \Comment{\eqref{eq:objectposition}}
\State $\mathbf{u}_i \gets -k(\mathbf{x}_i-\mathbf{p}_i^*)$ \Comment{Virtual commands}
\State $\mathbf{u}_{ku} \gets \mathbf{T}_i^{-1} \, \mathbf{u}_i$ \Comment{Actual commands}
\State $\textbf{COMMAND} \quad \mathbf{u}_{ku}$
\doWhile{$||\mathbf{x}_i-\mathbf{p}_i^*||>\delta$}
\State $g \gets 0$ \Comment{Disengage gripper}
\State $\textbf{return}$
\EndFunction
\end{algorithmic}
\caption{Position object $i$ to $\mathbf{p}_i^*$ along path $\mathcal{P}$}\label{algorithm:positionobject}
\end{algorithm}

\subsection{Action $\textsc{Move}$}
This action is exactly like $\textsc{MoveToObject}$, but there is no final orientation requirement. Its low-level implementation is shown in Algorithm \ref{algorithm:move}.

\begin{algorithm}
\begin{algorithmic}[1]
\Function{Move}{$\mathcal{P}$}
\State $\epsilon \gets \text{Set Wall Following Tolerance}$ \Comment{$\epsilon<\eta$}
\State $\delta \gets \text{Set Placement Tolerance}$
\State $\textsc{NavigateRobot}(\mathcal{P},r,\epsilon,\delta)$
\State $\textbf{return}$
\EndFunction
\end{algorithmic}
\caption{Free robot navigation along path $\mathcal{P}$}\label{algorithm:move}
\end{algorithm}

Note here that the formal results accompanying both the path following phase \cite{arslan_kod_ICRA2017} and the wall following phase (Theorems \ref{theorem:wallfollowing} and \ref{theorem:wallfollowingunicycle}) along with Theorem \ref{theorem:flow} guarantee that every symbolic action command will be successfully executed.

\section{NUMERICAL EXAMPLES}
\label{sec:simulations}

In this Section, we provide numerical examples\footnote{All simulations were run in MATLAB using \texttt{ode45} and a gain $k=2$.} of assembly processes in various environments using the symbolic action commands described above.

\subsection{Environment packed circular obstacles}
First, we test our algorithm in a rectangular, 20x20m workspace, packed with circular obstacles, whose position and size are unknown to the deliberative planner. The minimum separation $\eta$ between the obstacles is chosen to be only slightly above (e.g 5cm) the minimum allowed value prescribed by Assumption \ref{assumption:obstacle_separation}, in order to demonstrate the validity of our approach, deriving from the formal guarantees of Theorem \ref{theorem:wallfollowing}. The goal is to place an object to a desired position, shown in \ref{fig:packed_assembly}. The deliberative planner outputs a plan comprising of two actions: $\textsc{MoveToObject}(1,\mathcal{P}_1) \rightarrow \textsc{PositionObject}(1,\mathcal{P}_2)$, whose sequential execution and the corresponding reference paths $\mathcal{P}_1,\mathcal{P}_2$ are depicted in Fig. \ref{fig:packed_assembly}.

\begin{figure}[t]
\centering
\begin{subfigure}[t]{0.5\columnwidth}
\includegraphics[width=1.\columnwidth]{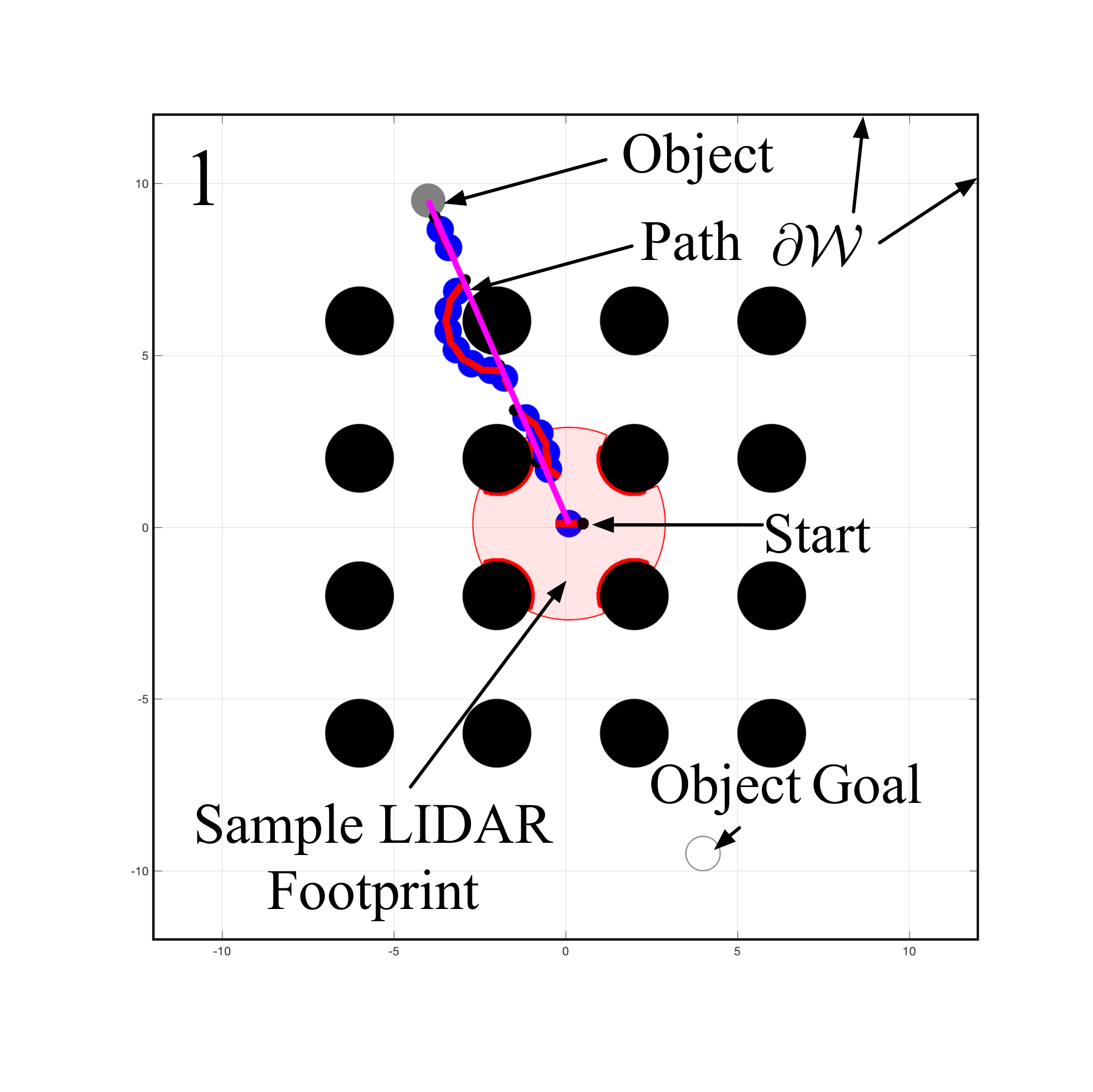}
\end{subfigure}%
\begin{subfigure}[t]{0.5\columnwidth}
\includegraphics[width=1.\columnwidth]{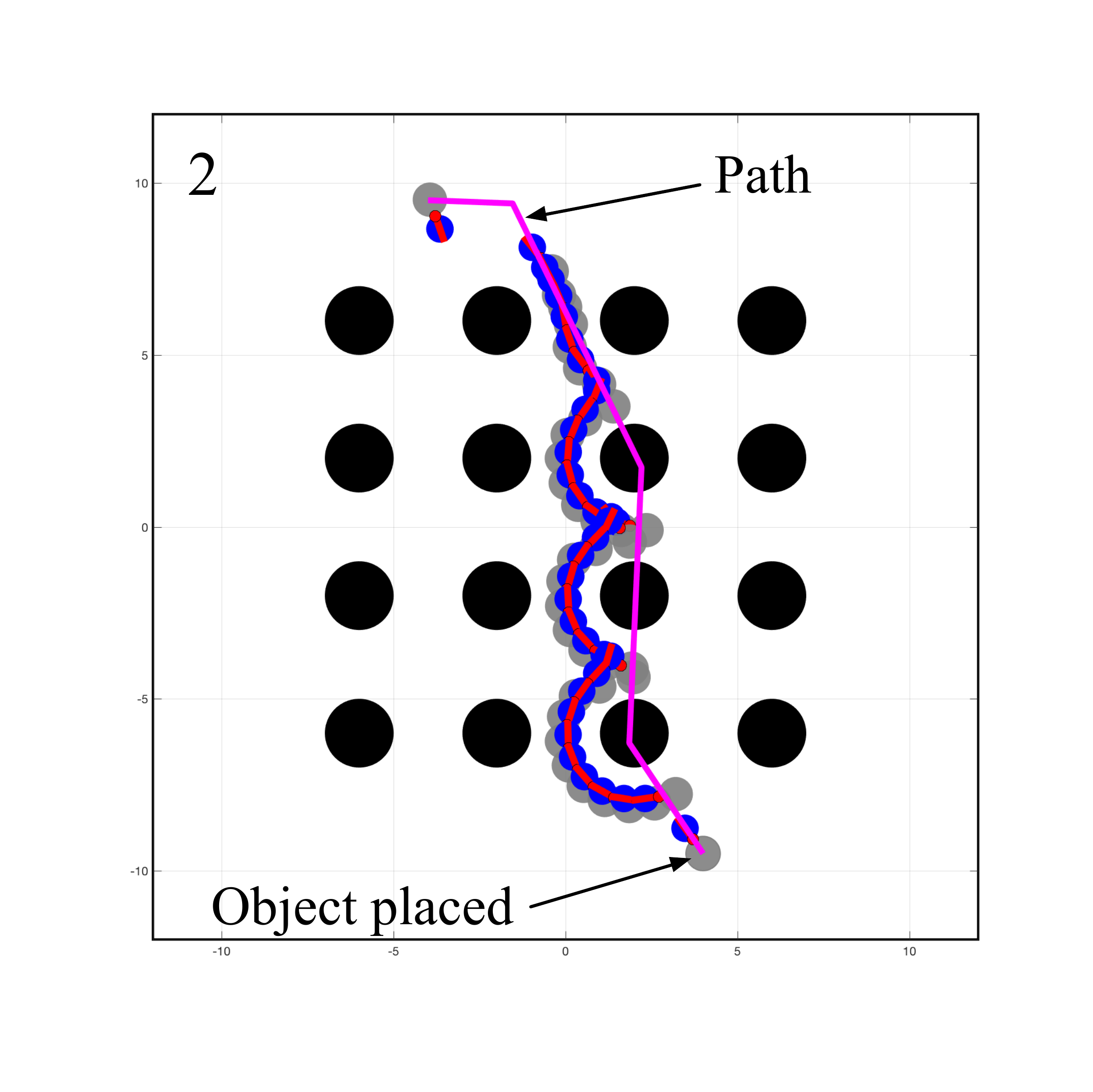}
\end{subfigure}
\caption{A depiction of a packed two stage assembly process with a fixed timestep, with the separation value just above the minimum allowed value.}
\label{fig:packed_assembly}
\end{figure}

\begin{figure}
\centering
\begin{subfigure}{0.5\columnwidth}
\includegraphics[width=1.\columnwidth]{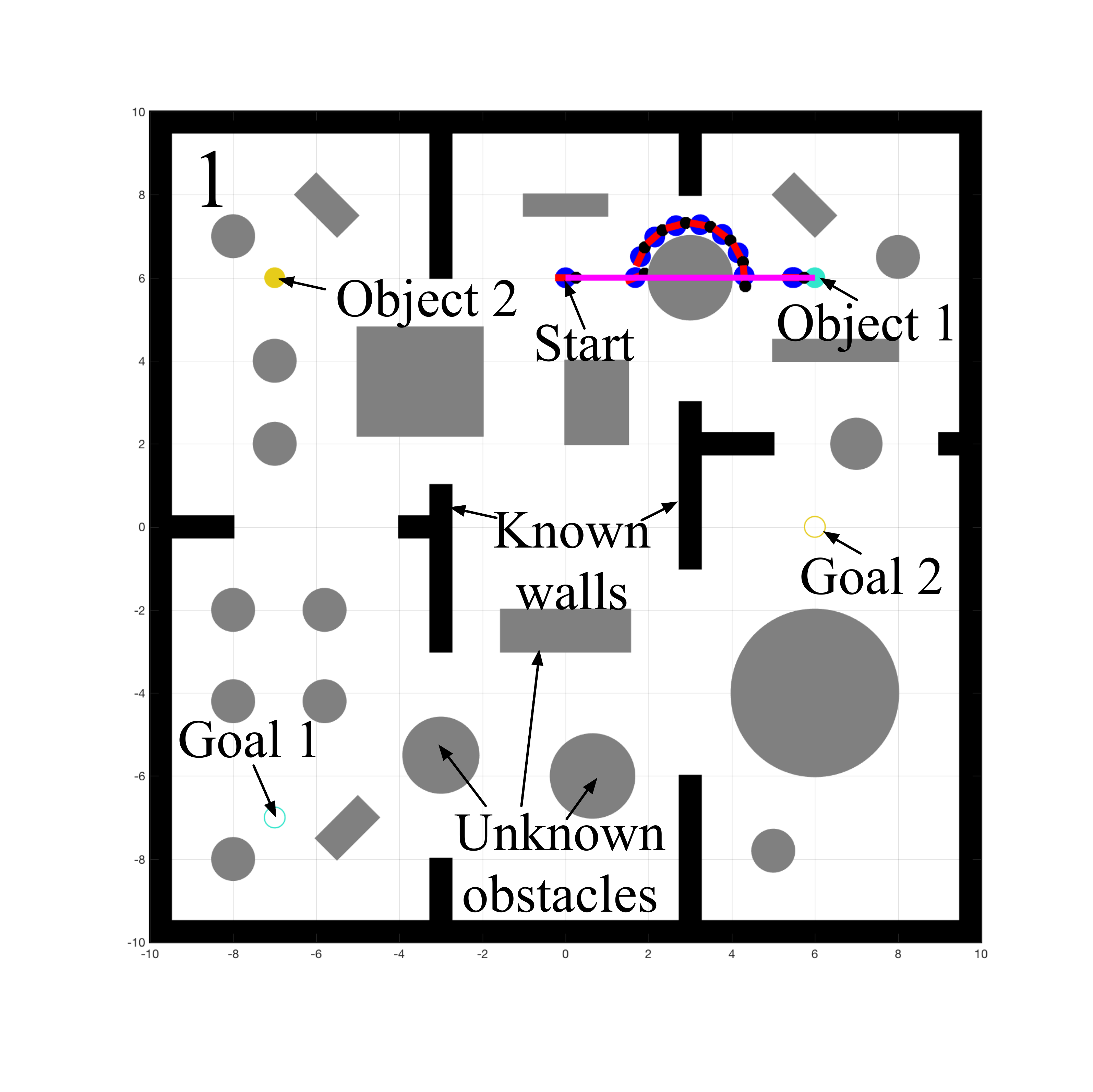}
\end{subfigure}%
\begin{subfigure}{0.5\columnwidth}
\includegraphics[width=1.\columnwidth]{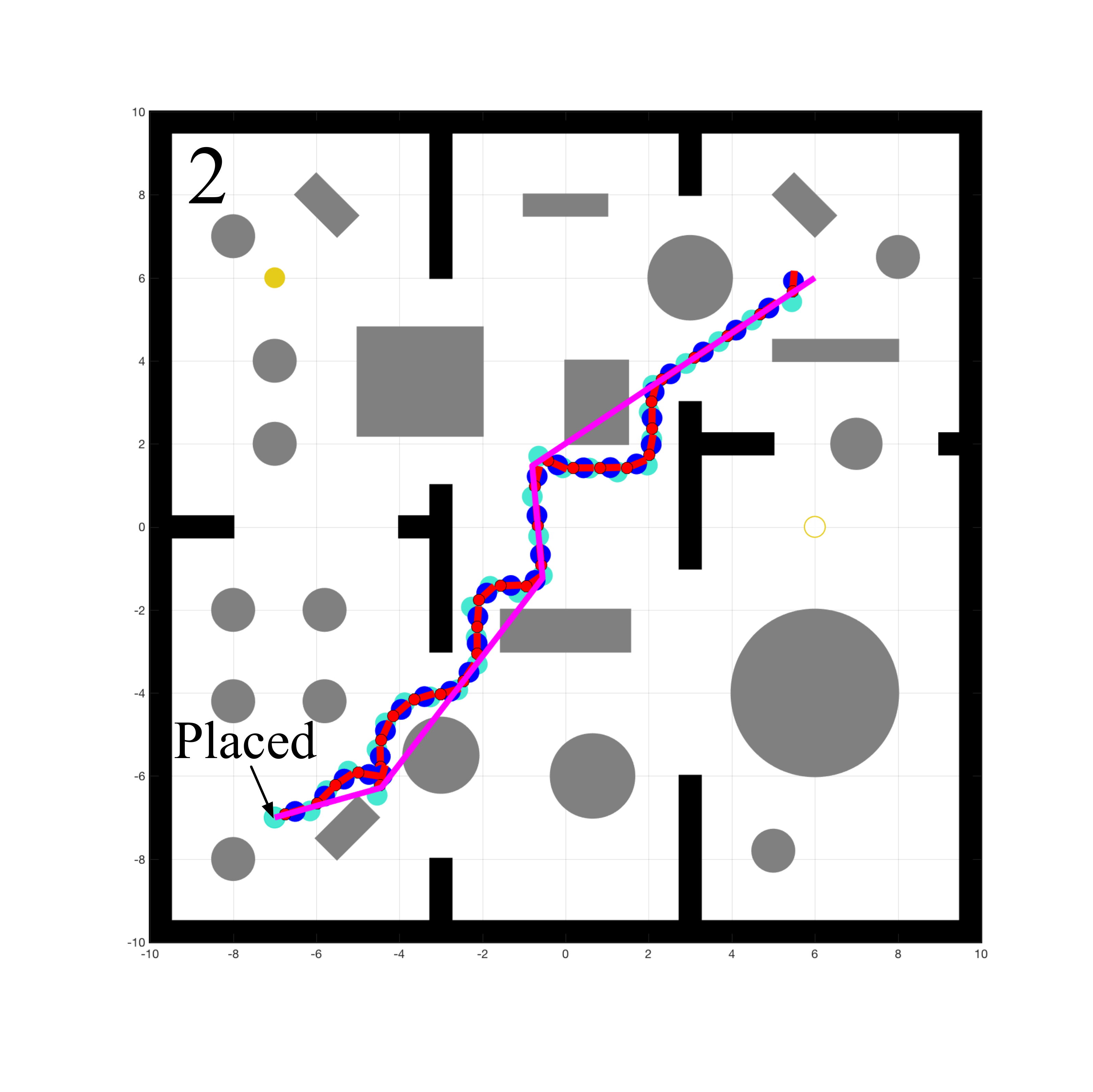}
\end{subfigure}
\begin{subfigure}{0.5\columnwidth}
\includegraphics[width=1.\columnwidth]{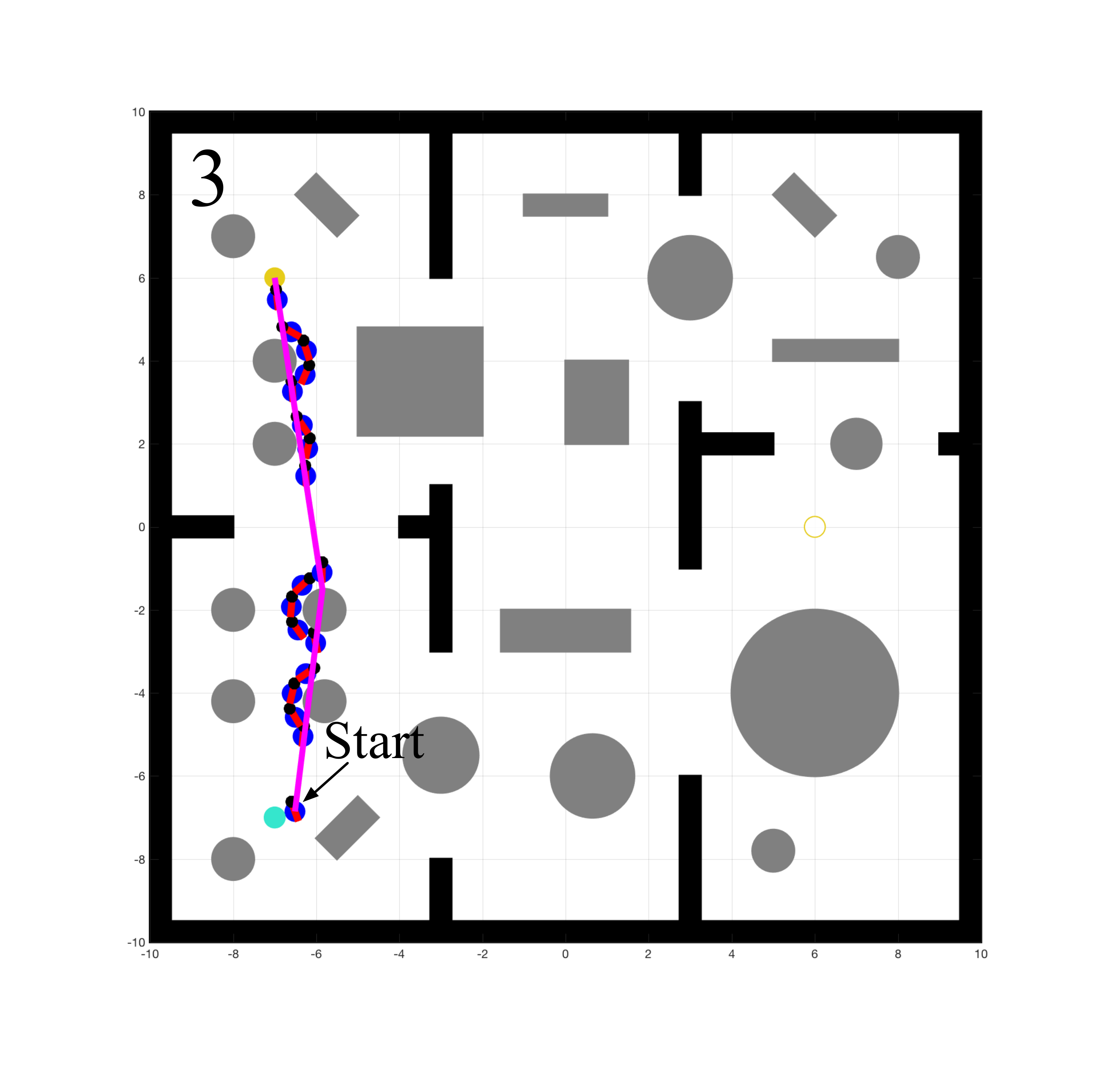}
\end{subfigure}%
\begin{subfigure}{0.5\columnwidth}
\includegraphics[width=1.\columnwidth]{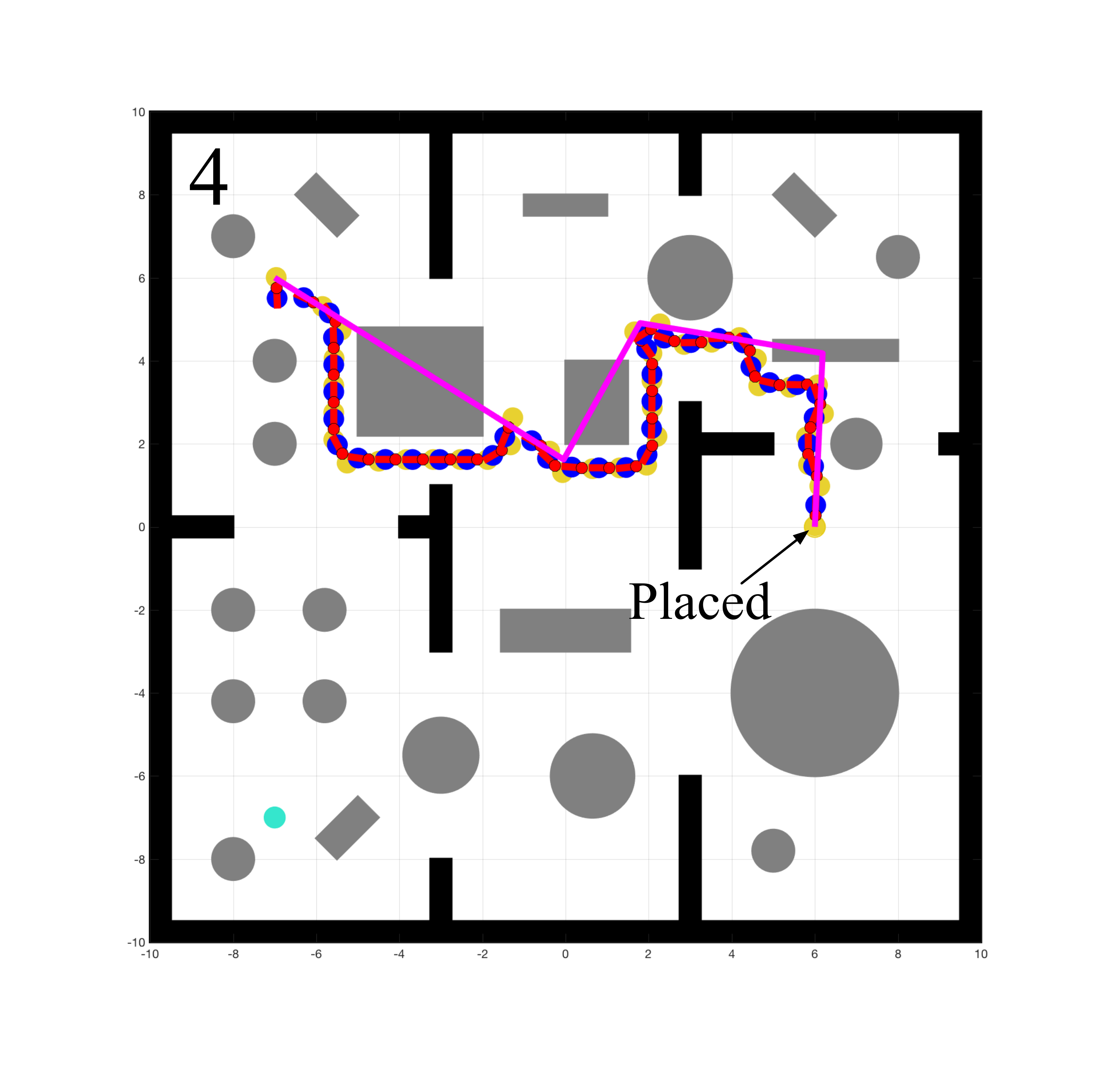}
\end{subfigure}
\begin{subfigure}{0.5\columnwidth}
\includegraphics[width=1.\columnwidth]{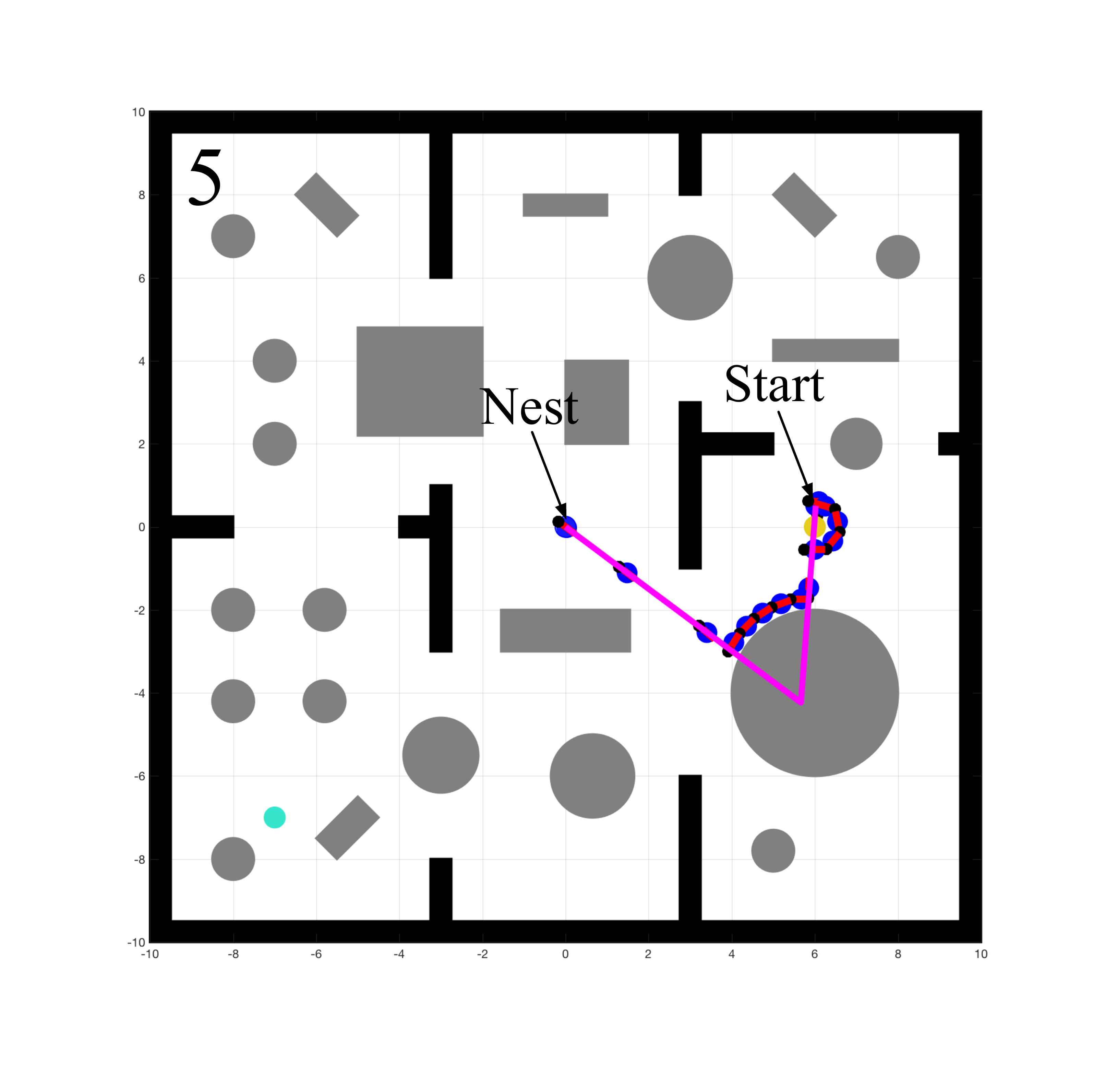}
\end{subfigure}%
\begin{subfigure}{0.5\columnwidth}
\includegraphics[width=1.\columnwidth]{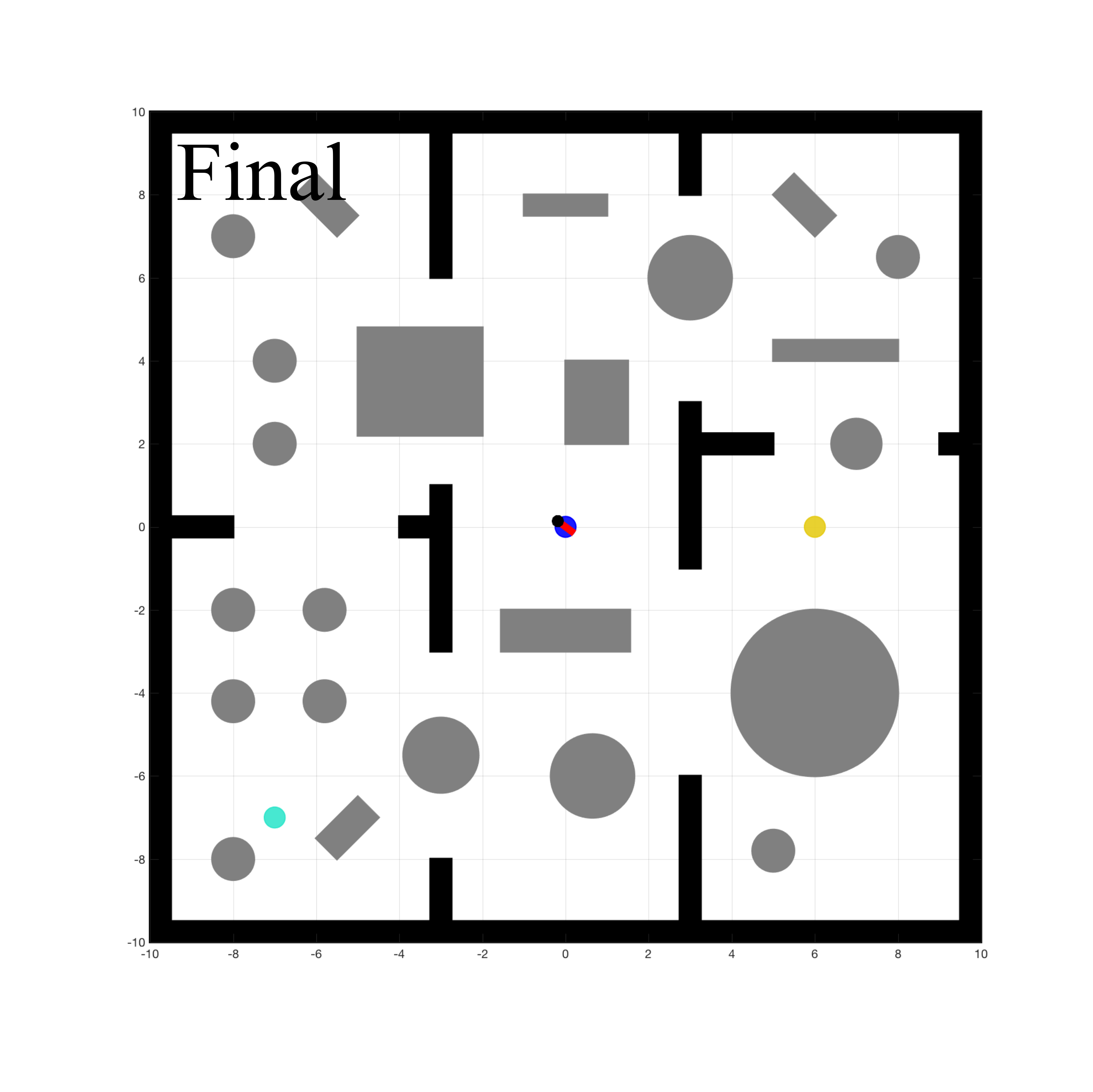}
\end{subfigure}
\caption{An illustration of the assembly process described in Section \ref{subsec:cluttered_environment}, with a fixed timestep. The walls and boundaries of the workspace, known to the deliberative planner, are shown in black and the unexpected obstacles handled by the reactive planner are shown in grey.}
\label{fig:assembly_walls}
\end{figure}

\subsection{Cluttered environment with walls}
\label{subsec:cluttered_environment}

In this paragraph, we demonstrate the execution of a more challenging task. The robot should position the two obstacles depicted in Fig. \ref{fig:assembly_walls} to their predefined positions within a polygonal workspace with walls, whose locations are provided a-priori to the deliberative planner, and then return to a ``nest'' location. The workspace is packed with several convex, not-necessarily circular obstacles. As shown in Fig. \ref{fig:assembly_walls}, the deliberative planner outputs a high-level plan comprising of five actions: $\textsc{MoveToObject}(1,\mathcal{P}_1) \rightarrow \textsc{PositionObject}(1,\mathcal{P}_2) \rightarrow \textsc{MoveToObject}(2,\mathcal{P}_3) \rightarrow \textsc{PositionObject}(2,\mathcal{P}_4) \rightarrow \textsc{Move}(\mathcal{P}_5)$, which is successfully executed by the reactive planner. An example for an object trajectory during this execution is shown in Fig. \ref{fig:environment}. Notice that, in contrast with several reactive wall following schemes that require an estimate of the wall curvature, our scheme can easily handle obstacles with corners. It is also worth noting that the deliberative planner hit the maximum number of expansions allowed and had difficulties extracting a feasible plan when it was provided the exact position and size of every obstacle, due to the highly packed construction. This highlights another benefit of our approach; we can significantly reduce the computational load of high-level planners by tasking them only with the extraction of the action sequence required, and using the reactive planner for local obstacle avoidance and convergence online. This happens because the computational load of the reactive planner remains the same regardless of the number of obstacles. 

Finally, it is worth noting that the proposed scheme is capable of executing a sequence of symbolic commands provided by the deliberative planner, even when Assumptions \ref{assumption:obstacle_separation} or \ref{assumption:admissible_goals} or the obstacle convexity are not satisfied. In the accompanying video, we provide examples of successful assemblies even in the absence of both obstacle convexity and enough separation.

\section{CONCLUSION AND FUTURE WORK}
\label{sec:conclusion}

This paper presents a provably correct architecture for planning and executing a successful solution to an instance of the Warehouseman's problem by decomposition into an offline ``deliberative'' planning module and an online ``reactive'' execution module. A differential drive robot equipped with a gripper and a LIDAR sensor, capable of perceiving its environment only locally, is used to successfully position the passive objects in a desired configuration with no collisions along the way. Formal proofs and numerical results demonstrate the validity of our method.

Future work will address the problem of closely integrating the reactive and deliberative planners and shifting more of the planning burden to the reactive component, allowing the deliberative planner to consider more complex dynamics and generate paths faster and more efficiently. Moreover, we plan to experimentally verify the reactive planning ideas for pushing a grasped object, as presented in Section \ref{sec:models}.

\addtolength{\textheight}{-12cm}  




\section*{ACKNOWLEDGMENT}
This work was supported in part by the ARL/GDRS RCTA project, Coop. Agreement \#W911NF-10–2−0016 and in part by AFRL grant FA865015D1845 (subcontract 669737–1).


\bibliographystyle{IEEEtran}
\bibliography{IEEEabrv,references}

\begin{thebibliography}{10}
\providecommand{\url}[1]{#1}
\csname url@rmstyle\endcsname
\providecommand{\newblock}{\relax}
\providecommand{\bibinfo}[2]{#2}
\providecommand\BIBentrySTDinterwordspacing{\spaceskip=0pt\relax}
\providecommand\BIBentryALTinterwordstretchfactor{4}
\providecommand\BIBentryALTinterwordspacing{\spaceskip=\fontdimen2\font plus
\BIBentryALTinterwordstretchfactor\fontdimen3\font minus
  \fontdimen4\font\relax}
\providecommand\BIBforeignlanguage[2]{{%
\expandafter\ifx\csname l@#1\endcsname\relax
\typeout{** WARNING: IEEEtran.bst: No hyphenation pattern has been}%
\typeout{** loaded for the language `#1'. Using the pattern for}%
\typeout{** the default language instead.}%
\else
\language=\csname l@#1\endcsname
\fi
#2}}

\bibitem{Hopcroft_Schwartz_Sharir_1984}
\BIBentryALTinterwordspacing
J.~E. Hopcroft, J.~T. Schwartz, and M.~Sharir, ``On the complexity of motion
  planning for multiple independent objects; pspace-hardness of the"
  warehouseman’s problem",'' \emph{The International Journal of Robotics
  Research}, vol.~3, no.~4, p. 76–88, 1984. [Online]. Available:
  \url{http://ijr.sagepub.com/content/3/4/76.short}
\BIBentrySTDinterwordspacing

\bibitem{Wolfe2010}
J.~Wolfe, B.~Marthi, and S.~Russell, ``Combined task and motion planning for
  mobile manipulation.'' in \emph{Proceedings of the International Conference
  on Automated Planning and Scheduling}, 2010.

\bibitem{Berenson2009}
D.~Berenson, S.~Srinivasa, D.~Ferguson, and J.~Kuffner, ``Manipulation planning
  on constraint manifolds,'' in \emph{Proceedings of the IEEE International
  Conference on Robotics and Automation (ICRA)}, 2009.

\bibitem{Konidaris2014}
G.~Konidaris, L.~Kaelbling, and T.~Lozano-Perez, ``Constructing symbolic
  representations for high-level planning,'' in \emph{proc. AAAI}, 2014.

\bibitem{Kaelbling2011}
L.~P. Kaelbling and T.~Lozano-Perez, ``Hierarchical task and motion planning in
  the now.'' in \emph{Proceedings of the IEEE International Conference on
  Robotics and Automation}, 2011.

\bibitem{WAFR16WRVB}
\BIBentryALTinterwordspacing
W.~Vega-Brown and N.~Roy, ``Asymptotically optimal planning under
  piecewise-analytic constraints,'' in \emph{Proceedings of the Workshop on the
  Algorithmic Foundations of Robotics}, San Francisco, CA, 2016. [Online].
  Available:
  \url{https://groups.csail.mit.edu/rrg/papers/vegabrown2016asymptotically.pdf}
\BIBentrySTDinterwordspacing

\bibitem{Vendittelli_Laumond_Mishra_2015}
\BIBentryALTinterwordspacing
M.~Vendittelli, J.-P. Laumond, and B.~Mishra, \emph{Decidability of Robot
  Manipulation Planning: Three Disks in the Plane}, ser. Springer Tracts in
  Advanced Robotics.\hskip 1em plus 0.5em minus 0.4em\relax Springer, Cham,
  2015, p. 641–657. [Online]. Available:
  \url{https://link.springer.com/chapter/10.1007/978-3-319-16595-0_37}
\BIBentrySTDinterwordspacing

\bibitem{Deshpande_Kaelbling_Lozano-Perez_2016}
\BIBentryALTinterwordspacing
A.~Deshpande, L.~P. Kaelbling, and T.~Lozano-Perez, \emph{Decidability of
  semi-holonomic prehensile task and motion planning}, 2016. [Online].
  Available: \url{http://lis.csail.mit.edu/pubs/deshpande-WAFR16.pdf}
\BIBentrySTDinterwordspacing

\bibitem{Chatila_1995}
\BIBentryALTinterwordspacing
R.~Chatila, ``Deliberation and reactivity in autonomous mobile robots,''
  \emph{Robotics and Autonomous Systems}, vol.~16, no.~2, p. 197–211, Dec
  1995. [Online]. Available:
  \url{http://www.sciencedirect.com/science/article/pii/0921889096810098}
\BIBentrySTDinterwordspacing

\bibitem{Koditschek_1994}
D.~E. Koditschek, ``An approach to autonomous robot assembly,''
  \emph{Robotica(Cambridge. Print)}, vol.~12, p. 137–155, 1994.

\bibitem{Bozma_Koditschek_2001}
H.~Isil~Bozma and D.~E. Koditschek, ``Assembly as a noncooperative game of its
  pieces: analysis of 1d sphere assemblies,'' \emph{Robotica}, vol.~19, no.~01,
  p. 93–108, 2001.

\bibitem{Karagoz_Bozma_Koditschek_2004}
C.~S. Karagoz, H.~I. Bozma, and D.~E. Koditschek, ``Feedback-based event-driven
  parts moving,'' \emph{Robotics, IEEE Transactions on [see also Robotics and
  Automation, IEEE Transactions on]}, vol.~20, no.~6, p. 1012–1018, 2004.

\bibitem{marthi2008angelic}
B.~Marthi, S.~Russell, and J.~Wolfe, ``Angelic hierarchical planning: Optimal
  and online algorithms,'' in \emph{ICAPS}, 2008.

\bibitem{vegabrown2017optimal}
\BIBentryALTinterwordspacing
W.~Vega-Brown and N.~Roy, ``Optimal and near-optimal task and motion planning
  with abstraction,'' 2017. [Online]. Available:
  \url{https://groups.csail.mit.edu/rrg/papers/vegabrown2017optimal.pdf}
\BIBentrySTDinterwordspacing

\bibitem{koditschek-aam-1990}
D.~Koditschek and E.~Rimon, ``Robot navigation functions on manifolds with
  boundary,'' \emph{Advances in Applied Mathematics}, vol.~11, no.~4, pp.
  412--442, 1990.

\bibitem{arslan_kod_ICRA2016B}
O.~Arslan and D.~E. Koditschek, ``Exact robot navigation using power
  diagrams,'' in \emph{Robotics and Automation, 2016 IEEE International
  Conference on}, 2016, pp. 1--8.

\bibitem{arslan_kod_WAFR2016}
------, ``Sensor-based reactive navigation in unknown convex sphere worlds,''
  in \emph{The 12th International Workshop on the Algorithmic Foundations of
  Robotics}, 2016.

\bibitem{arslan_kod_ICRA2017}
------, ``Smooth extensions of feedback motion planners via reference
  governors,'' in \emph{Robotics and Automation, 2017 IEEE International
  Conference on}, 2017(accepted).

\bibitem{choset_lynch_hutchinson_kantor_burgard_kavraki_thrun_2005}
H.~Choset, K.~M. Lynch, S.~Hutchinson, G.~Kantor, W.~Burgard, L.~Kavraki, and
  S.~Thrun, \emph{Principles of robot motion: theory, algorithms, and
  implementations.}\hskip 1em plus 0.5em minus 0.4em\relax MIT Press,
  Cambridge, MA, 2005.

\bibitem{Kuntz-1994}
L.~Kuntz and S.~Scholtes, ``Structural analysis of nonsmooth mappings, inverse
  functions, and metric projections,'' \emph{Journal of Mathematical Analysis
  and Applications}, vol. 188, no.~2, pp. 346 -- 386, 1994.

\bibitem{shapiro-1988}
A.~Shapiro, ``Sensitivity analysis of nonlinear programs and differentiability
  properties of metric projections,'' \emph{SIAM Journal on Control and
  Optimization}, vol.~26, no.~3, pp. 628--645, 1988.

\bibitem{chaney-1990}
R.~W. Chaney, ``Piecewise ck functions in nonsmooth analysis,'' \emph{Nonlinear
  Analysis: Theory, Methods \& Applications}, vol.~15, no.~7, pp. 649 -- 660,
  1990.

\bibitem{burridge-ijrr-1999}
R.~Burridge, A.~Rizzi, and D.~Koditschek, ``Sequential composition of
  dynamically dexterous robot behaviors,'' \emph{The International Journal of
  Robotics Research}, vol.~18, pp. 534--555, 1999.

\bibitem{Brockett-83}
R.~W. Brockett, ``Asymptotic stability and feedback stabilization,'' in
  \emph{Differential Geometric Control Theory}.\hskip 1em plus 0.5em minus
  0.4em\relax Birkhauser, 1983, pp. 181--191.

\end{thebibliography}

\end{document}